\documentclass[10pt,journal,compsoc]{IEEEtran}

\usepackage{amsmath}

\usepackage{cite}
\usepackage{amsmath,amssymb,amsfonts}
\allowdisplaybreaks

\usepackage{array}
\usepackage{graphicx}
\usepackage{stfloats}
\usepackage{amsthm}
\usepackage{caption}
\usepackage{subfigure}
\usepackage{url}
\usepackage{verbatim}  
\usepackage{textcomp}
\usepackage{xcolor}
\usepackage{color}
\usepackage{enumitem}
\usepackage{algorithmicx}
\usepackage[noend,ruled, linesnumbered, vlined]{algorithm2e}
\usepackage{algpseudocode}
\usepackage[absolute,overlay]{textpos}

\allowdisplaybreaks[4]
\def\BibTeX{{\rm B\kern-.05em{\sc i\kern-.025em b}\kern-.08em
    T\kern-.1667em\lower.7ex\hbox{E}\kern-.125emX}}
\newenvironment{sequation}{\begin{equation}\small}{\end{equation}}

\newtheorem{theorem}{\textbf{Theorem}}
\newtheorem{remark}{\textbf{Remark}}
\newtheorem{definition}{\textbf{Definition}}

\usepackage{multirow}
\usepackage{bm}
\usepackage{booktabs}

\usepackage{epstopdf}
\usepackage{cases}
\usepackage{makecell,multirow,diagbox}
\usepackage{stfloats}
\usepackage{comment}
\usepackage{array}

\usepackage{xr}
\makeatletter
\newcommand*{\addFileDependency}[1]{
  \typeout{(#1)}
  \@addtofilelist{#1}
  \IfFileExists{#1}{}{\typeout{No file #1.}}
}
\makeatother

\newcommand*{\myexternaldocument}[1]{
    \externaldocument{#1}
    \addFileDependency{#1.tex}
    \addFileDependency{#1.aux}
}

\myexternaldocument{Appendices}

\definecolor{b}{rgb}{0.0, 0, 1}
\definecolor{k}{rgb}{0, 0, 0}

\def\BibTeX{{\rm B\kern-.05em{\sc i\kern-.025em b}\kern-.08em
		T\kern-.1667em\lower.7ex\hbox{E}\kern-.125emX}}

\begin{document}

\title{Cost Minimization for Space-Air-Ground Integrated Multi-Access Edge Computing Systems}

\author{Weihong~Qin,
        Aimin Wang,
        Geng~Sun,~\IEEEmembership{Senior Member,~IEEE},
    Zemin~Sun,~\IEEEmembership{Member,~IEEE},\\
        Jiacheng Wang,
        Dusit Niyato,~\IEEEmembership{Fellow,~IEEE},
        Dong In Kim,~\IEEEmembership{Fellow,~IEEE},
        Zhu Han,~\IEEEmembership{Fellow,~IEEE}

    \thanks{This study is supported in part by the National Natural Science Foundation of China (62272194, 62471200, 62502179), and in part by the Science and Technology Development Plan Project of Jilin Province (20250101027JJ). (\textit{Corresponding authors: Geng Sun and Zemin Sun.})}
    \IEEEcompsocitemizethanks{
    \IEEEcompsocthanksitem Weihong Qin, Aimin Wang and Zemin Sun are with the College of Computer Science and Technology, Key Laboratory of Symbolic Computation and Knowledge Engineering of Ministry of Education, Jilin University, Changchun 130012, China (e-mail: qinwh25@mails.jlu.edu.cn, wangaimin@jlu.edu.cn, sunzemin@jlu.edu.cn).
    
    \IEEEcompsocthanksitem Geng Sun is with the College of Computer Science and Technology, Key Laboratory of Symbolic Computation and Knowledge Engineering of Ministry of Education, Jilin University, Changchun 130012, China, and also with the College of Computing and Data Science, Nanyang Technological University, Singapore 639798 (e-mail: sungeng@jlu.edu.cn).

    \IEEEcompsocthanksitem Jiacheng Wang and Dusit Niyato are with the College of Computing and Data Science, Nanyang Technological University, Singapore 639798 (e-mail: jiacheng.wang@ntu.edu.sg, dniyato@ntu.edu.sg).

    \IEEEcompsocthanksitem Dong In Kim is with the Department of Electrical and Computer Engineering, Sungkyunkwan University, Suwon 16419, South Korea. (e-mail:dongin@skku.edu).

    \IEEEcompsocthanksitem Zhu Han is with the Department of Electrical and Computer Engineering at the University of Houston, Houston TX 77004, USA, and also with the Department of Computer Science and Engineering, Kyung Hee University, Seoul 446701, South Korea (e-mail: hanzhu22@gmail.com).

    }
}

\IEEEtitleabstractindextext{%
\begin{abstract}		
\par Space-air-ground integrated multi-access edge computing (SAGIN-MEC) provides a promising solution for the rapidly developing low-altitude economy (LAE) to deliver flexible and wide-area computing services. However, fully realizing the potential of SAGIN-MEC in the LAE presents significant challenges, including coordinating decisions across heterogeneous nodes with different roles, modeling complex factors such as mobility and network variability, and handling real-time decision-making under partially observable environment with hybrid variables. To address these challenges, we first present a hierarchical SAGIN-MEC architecture that enables the coordination between user devices (UDs), uncrewed aerial vehicles (UAVs), and satellites. Then, we formulate a UD cost minimization optimization problem (UCMOP) to minimize the UD cost by jointly optimizing the task offloading ratio, UAV trajectory planning, computing resource allocation, and UD association. We show that the UCMOP is an NP-hard problem. To overcome this challenge, we propose a multi-agent deep deterministic policy gradient (MADDPG)-convex optimization and coalitional game (MADDPG-COCG) algorithm. Specifically, we employ the MADDPG algorithm to optimize the continuous temporal decisions for heterogeneous nodes in the partially observable SAGIN-MEC system. Moreover, we propose a convex optimization and coalitional game (COCG) method to enhance the conventional MADDPG by deterministically handling the hybrid and varying-dimensional decisions. Simulation results demonstrate that the proposed MADDPG-COCG algorithm significantly enhances the user-centric performances in terms of the aggregated UD cost, task completion delay, and UD energy consumption, with a slight increase in UAV energy consumption, compared to the benchmark algorithms. Moreover, the MADDPG-COCG algorithm shows superior convergence stability and scalability.
\end{abstract}

}

\maketitle
\IEEEdisplaynontitleabstractindextext
\IEEEpeerreviewmaketitle

%
%

\section{Introduction}
\label{sec_introduction}
\par \IEEEPARstart{A}{s} a nascent field, the low-altitude economy (LAE) is quickly becoming a rapidly developing sector in modern economic systems, driven by technological advancements and regulatory reforms that have enabled the commercial exploitation of underutilized airspace. Supported by the deployment of uncrewed aerial vehicles (UAVs), including drones and electric vertical take-off and landing aircraft, the LAE involves a broad range of commercial and public service activities conducted within low-altitude airspace, such as logistics delivery, emergency response, environmental monitoring, and precision agriculture \cite{Ye2024}. These services not only enhance the service efficiency but also create new commercial opportunities across a wide range of industries. The leading enterprises have already begun to capitalize on this emerging market. For example, companies such as DJI and EHang have actively explored diverse low-altitude applications, including drone-based imaging and autonomous aerial transportation \cite{dai2024data}. 

\par The expansion of LAE has seen the widespread deployment of UAVs across multiple domains by leveraging their mobility, flexibility, and maneuverability \cite{li2024multi}. Among these applications, UAV-assisted multi-access edge computing (MEC) stands out as a highly potential solution, which involves rapidly deploying UAVs to act as MEC servers nearby ground user devices (UDs). By offloading the delay-sensitive and resource-intensive tasks, such as navigation, road traffic monitoring, and rescue operations, to nearby UAVs, the computational and energy constraints of UDs can be alleviated~\cite{Sun2024TJCCT, he2024online, sun2024joint_rescue}. This is particularly advantageous in remote, infrastructure-scarce, or disaster-affected areas where ground base stations may be unavailable or damaged. Nonetheless, the practical utility of UAVs is hampered by inherent deficiencies in onboard computing, energy, and communication range, all of which impede the delivery of reliable, long-term services. 

\par To address the limitations of UAVs, combining the space-air-ground integrated network with MEC presents a viable approach, leading to the SAGIN-MEC framework. SAGIN-MEC leverages the rapid deployment of low earth orbit (LEO) satellite, such as Starlink \cite{mcdowell2020low}, OneWeb \cite{radtke2017interactions}, and Kuiper \cite{liu2020automatic}, to integrate terrestrial, aerial, and space-based networks into a unified architecture. On the one hand, SAGIN-MEC offers seamless computation offloading across a wide coverage area by leveraging the wide coverage of LEO satellites and the mobility of UAVs. On the other hand, SAGIN-MEC effectively integrates space and edge computing resources to improve the resource utilization.

\par However, fully exploiting the potential of the SAGIN-MEC system presents several challenges. \textbf{First}, the SAGIN-MEC system is inherently dynamic, with mobile nodes, time-varying channels, fluctuating workloads, and varying requirements of UDs, UAVs, and LEO satellites. Therefore, it is challenging to capture these dynamics to ensure long-term, efficient, and stable performances. \textbf{Second}, the SAGIN-MEC system consists of heterogeneous nodes, i.e., UDs, UAVs, and LEO satellites, which have distinct requirements and resources. However, although these nodes make decisions independently, their decisions jointly affect the overall system performance. Consequently, ensuring that heterogeneous nodes can make distributed decisions while coordinating the decisions across space, air, and ground layers to achieve efficient overall system performance remains a significant challenge. \textbf{Finally}, in realistic SAGIN-MEC scenarios, UDs and UAVs can obtain only partial local information due to mobility, coverage constraints, and communication delays, which inherently prevents full knowledge of the global state. Moreover, the decisions of different nodes exhibit complex characteristics, which combines varying dimensionality with discrete and continuous decision types. Consequently, the coexistence of incomplete information and varying-dimensional, hybrid decision structures makes achieving timely decisions particularly challenging for the SAGIN-MEC framework.

\par To ensure efficient long-term system performance in the dynamic SAGIN-MEC system, DRL (deep reinforcement learning) stands out as a promising algorithm. Moreover, the presence of heterogeneous nodes with diverse requirements and interdependent decisions necessitates a framework for distributed and coordinated decision-making, naturally motivating the use of multi-agent DRL (MADRL) algorithm. However, conventional MADRL algorithms struggle to handle the varying-dimensional and hybrid decisions. In response to these limitations, we presents a multi-agent deep deterministic policy gradient (MADDPG)-convex optimization and coalitional game (MADDPG-COCG) algorithm. Specifically, we develop the MADDPG to learn continuous temporal decisions. Moreover, we incorporate a convex optimization and coalitional game (COCG)-based method to enhance the conventional MADDPG by deterministically addressing the varying-dimensional and hybrid decisions. We highlight our main contributions as follows.

\begin{itemize}
	\item \textit{\textbf{System Architecture.}} We propose a hierarchical architecture for the SAGIN-MEC system in the LAE scenario. This architecture integrates a UD layer, an air layer, and a space layer. It provides integrated computing services by using UAVs for low-latency edge computation and the satellite network for cloud access. In this architecture, the heterogeneous nodes with different roles and constraints make distributed decisions that collectively optimize the overall system performance.
    
    \item \textit{\textbf{Problem Formulation.}} To capture the delay sensitivity of most tasks and energy constraints of UDs, we formulate a UD cost minimization optimization problem (UCMOP). The core objective is to curtail both task completion delay and UD energy consumption. This is achieved through the co-optimization of the task offloading ratio, UAV trajectory, computing resource allocation, and UD association, all while adhering to the resource and energy constraints of the UAVs. This problem is shown to be both non-convex and NP-hard, rendering it intractable for traditional optimization algorithms.
    
    \item \textbf{\textit{Algorithm Design.}} To solve the NP-hard UCMOP, we propose the MADDPG-COCG algorithm. Specifically, the MADDPG is employed to enable heterogeneous nodes (i.e., UDs and UAVs) to learn continuous temporal decisions of task offloading ratio and UAV trajectory planning under partial observability. Moreover, the COCG method is designed to enhance MADDPG by deterministically obtaining the varying-dimensional and discrete decisions. For the varying-dimensional decision on computing resource allocation, we leverage convex optimization to derive a closed-form expression. For the discrete decision of UD association, we obtain the stable and mutually beneficial result in a distributed manner by using the coalitional game.
    
   \item \textbf{\textit{Performance Evaluation.}} Simulation results validate the effectiveness of the proposed MADDPG-COCG algorithm. Specifically, the proposed MADDPG-COCG algorithm exhibits superior performance compared to two competing MADRL algorithms and two key ablation baselines, achieving significant improvements in user-centric metrics while maintaining relatively efficient UAV energy consumption. Furthermore, our algorithm exhibits faster convergence speed and enhanced learning stability when compared to MADDPG and MAPPO. To further verify its robustness and scalability, we conduct extensive performance analyses under diverse system settings, including a varying number of UDs, different task sizes, and fluctuating computing resources.
\end{itemize}

\par The remainder of this paper is organized as follows. Section \ref{sec_related work} reviews related work. Section \ref{sec_model} presents the system model. Section \ref{sec_problemFormulation} formulates the problem. Section \ref{sec_algorithm} provides a detailed description of the proposed algorithm. We show simulation results in Section \ref{sec_simulation} and conclude in Section \ref{sec_conclusion}.

%
%
\section{Related work}
\label{sec_related work}
\par In this section, we review the related work from the aspects of MEC architectures, formulation of joint optimization problems, and optimization approaches.

\subsection{MEC Architecture}

\par Considering that conventional MEC is limited by insufficient coverage and lack of flexibility due to the dependence on terrestrial infrastructures, numerous studies have focused on UAV-enabled MEC. For example, Pervez \emph{et al.}~\cite{Pervez2024} presents a hybrid air-ground architecture, in which a base station and multiple UAVs collaboratively provide computing services to users. Zhang \emph{et al.}~\cite{Zhang2024Task} architected a UAV-assisted  MEC architecture, where the considered region is divided into several sub-regions, and the tasks of ground users are offloaded to UAVs. Moreover, Bai \emph{et al.}~\cite{Bai2024} studied a hybrid edge-cloud model, incorporating aerial platforms with distant cloud infrastructure. However, these studies often overlook the fundamental limitation that UAVs have limited coverage radius \cite{sun2024uav}, which can lead to service discontinuity and insufficient service quality, especially in large-scale deployment scenarios. 

\par Due to the extensive coverage and system robustness, the integration of MEC with satellite systems has gained increasing attention recently. For example, Peng \emph{et al.}~\cite{Peng2025} introduced a satellite-cloud-MEC integrated vehicular network, where the LEO satellite relays the tasks from terrestrial MEC servers to the cloud server. Moreover, Li \emph{et al.}~\cite{Li2025Joint} considered an MEC-enabled LEO satellite network, allowing UD tasks to be handled either by the satellite itself or relayed to the cloud for execution. However, most of these studies focus on satellite-assisted task offloading for terrestrial MEC servers, which limits the applicability for delay sensitive tasks because of the significant propagation delay inherent in ground-satellite communication.

\par To leverage the advantages of both UAVs and satellites, SAGIN-MEC has attracted extensive attention. For example, Du \emph{et al.}~\cite{Du2024Joint}proposed an architecture in which multiple UAVs, alongside a single satellite, are utilized to offer computational support for Internet of Things (IoT) devices. Moreover, Tun \emph{et al.}~\cite{Tun2025Joint} considered a THz-assisted SAGIN-MEC system aiming for minimal energy drain on IoT devices. Zhang \emph{et al.}~\cite{Zhang2025Joint} explored an SAGIN-MEC-assisted task offloading and energy harvesting system, where an LEO satellite and a set of UAVs are deployed to offer energy harvesting and task offloading capabilities to users. However, most of these studies addressed relatively simple scenarios where a single satellite is deployed, UAVs fly along fixed trajectories, or UDs are assumed to be stationary.

\par In summary, the MEC architectures considered in existing studies are constrained by limited coverage and service continuity, are prone to high latency for delay-sensitive tasks, or oversimplify real-world deployment scenarios. This work proposes a hierarchical SAGIN-MEC architecture  that guarantees flexible and scalable computing services through efficient coordination among satellite, aerial, and ground nodes. 

\subsection{Formulation of Joint Optimization Problems}

\par The formulation of the optimization problem is fundamental to measuring the SAGIN-MEC system overall performance. Previous studies have explored the optimization objectives such as delay and energy consumption. For instance, Waqar \emph{et al.}~\cite{Waqar2022Computation} formulated a delay minimization problem for MEC-assisted aerial-terrestrial integrated vehicular networks. Similarly, Nguyen \emph{et al.}~\cite{nguyen2022joint} jointly optimized computation offloading, UAV trajectory, and resource allocation within a SAGIN. Their objective was to minimize a weighted sum composed of both task completion delay and energy consumption. Moreover, Nguyen \emph{et al.}~\cite{Nguyen2024} focus on task offloading in the hybrid edge-cloud SAGIN, where the objective is to reduce the weighted energy consumption for the entire SAGIN. However, these studies mainly focused on optimizing one perspective of the performance metrics, which may lead to suboptimal solutions that fail to meet the diverse requirements of tasks where both delay and energy efficiency are critical considerations.

\par Besides the optimization objective, the joint optimization of decision variables also significantly affects the system performance. Researchers have optimized various aspects of SAGIN-MEC such as resource allocation and task offloading. For example, Chen \emph{et al.}~\cite{Chen2023} jointly optimized access strategy, transmit power, task offloading, and resource allocation for UAV-augmented SAGIN. Moreover, Zhu \emph{et al.}~\cite{Zhu2025Resource} explored the joint optimization of resource allocation and task offloading for SAGIN-MEC systems. Furthermore, focusing on the emerging low-altitude economy, Tang \emph{et al.}~\cite{Tang2025} jointly optimized UAV exploration and task assignment to minimize both energy consumption and completion delay in a cooperative emergency rescue system. However, these studies lack a comprehensive consideration of the key factors such as task offloading ratio and UAV trajectory control.

\par This work differs from the aforementioned studies in several aspects. First, rather than focusing solely on a single metric, our optimization objective simultaneously considers both energy consumption and task completion delay to address these critical performance metrics. Moreover, we adopt a more comprehensive set of decision variables by jointly optimizing task offloading ratio, UAV trajectory planning, resource allocation, and UD association, aiming to fully exploit the benefits of the SAGIN-MEC system.

\subsection{Optimization Approach}

\par To address the complex optimization problems, researchers are devoted to designing efficient optimization approaches by employing methodologies such as optimization theory, heuristic algorithms, and evolutionary algorithms. For example, Zhou \emph{et al.}~\cite{Zhou2022} used Lyapunov optimization and dependent rounding within an alternating optimization framework to minimize service delay. Moreover, Laboni \emph{et al.}~\cite{Goudarzi2023} developed a hyper-heuristic algorithm for resource allocation in MEC networks. Additionally, evolutionary algorithms have also been widely adopted. For instance, Zhou \emph{et al.}~\cite{zhou2022multi} addressed the joint optimization of service caching and task offloading in MEC-assisted industrial IoTs by proposing a multi-objective evolutionary algorithm. He \emph{et al.}~\cite{he2021uav} addressed the task offloading problem in a UAV-assisted edge computing framework for vehicular networks by developing a genetic-algorithm-based approach. Furthermore, an improved NSGA-II algorithm was designed in~\cite{du2024jointoff} to address the joint offloading and resource management in multi-UAV-assisted MEC. However, the applicability of most of these traditional optimization methods is limited in highly dynamic environments. Furthermore, heuristic algorithms often exhibit sensitivity to parameter settings and may converge to local optima, while evolutionary algorithms can suffer from slow convergence and high computational costs.

\par To better adapt to the dynamic nature of SAGIN-MEC systems, DRL has emerged as an effective paradigm that can operate without prior information about the environment. For example, Hoang et al. ~\cite{hoang2023deep} integrated Lyapunov theory into a DRL framework to ensure stability in learning process. To improve energy efficiency for data gathering in wireless charging environments, the authors of ~\cite{Fu2021} utilized deep Q-learning to optimize UAV trajectory. However, traditional DRL approaches, which typically rely on a centralized agent, face significant challenges in coordinating the decisions of multiple heterogeneous nodes with distinct objectives. Consequently, MADRL has been adopted as a more suitable paradigm. For instance, Du \emph{et al.}~\cite{du2024multiagent} developed an MADRL algorithm for the co-optimization of resource allocation and task offloading in SAGINs. In~\cite{Du2024MADDPG}, the authors presented an MADDPG-based algorithm for service placement and task offloading in SAGIN-MEC systems. Moreover, Li \emph{et al.}~\cite{Li2025Collaborative} also proposed a multi-agent proximal policy optimization (MAPPO)-based approach to minimize UD energy consumption. Nevertheless, conventional MADRL algorithms still struggle to effectively address the challenges of hybrid discrete-continuous and varying-dimensional action spaces.

\par To address these specific challenges, this paper presents an MADDPG-COCG algorithm that integrates COCG method into the MADDPG framework, thus enhancing the capability of the conventional MADDPG to efficiently handle the varying-dimensional and hybrid decisions while achieving superior performance.

%
%
\section{System Model and Problem Formulation}
\label{sec_model}

\par In this section, we first introduce a SAGIN-MEC architecture, followed by the models of mobility, communication, computation, and cost.

\begin{figure}[t] 
	\centering
    \setlength{\abovecaptionskip}{2pt}%
	\setlength{\belowcaptionskip}{2pt}%
	\includegraphics[width =3.5in]{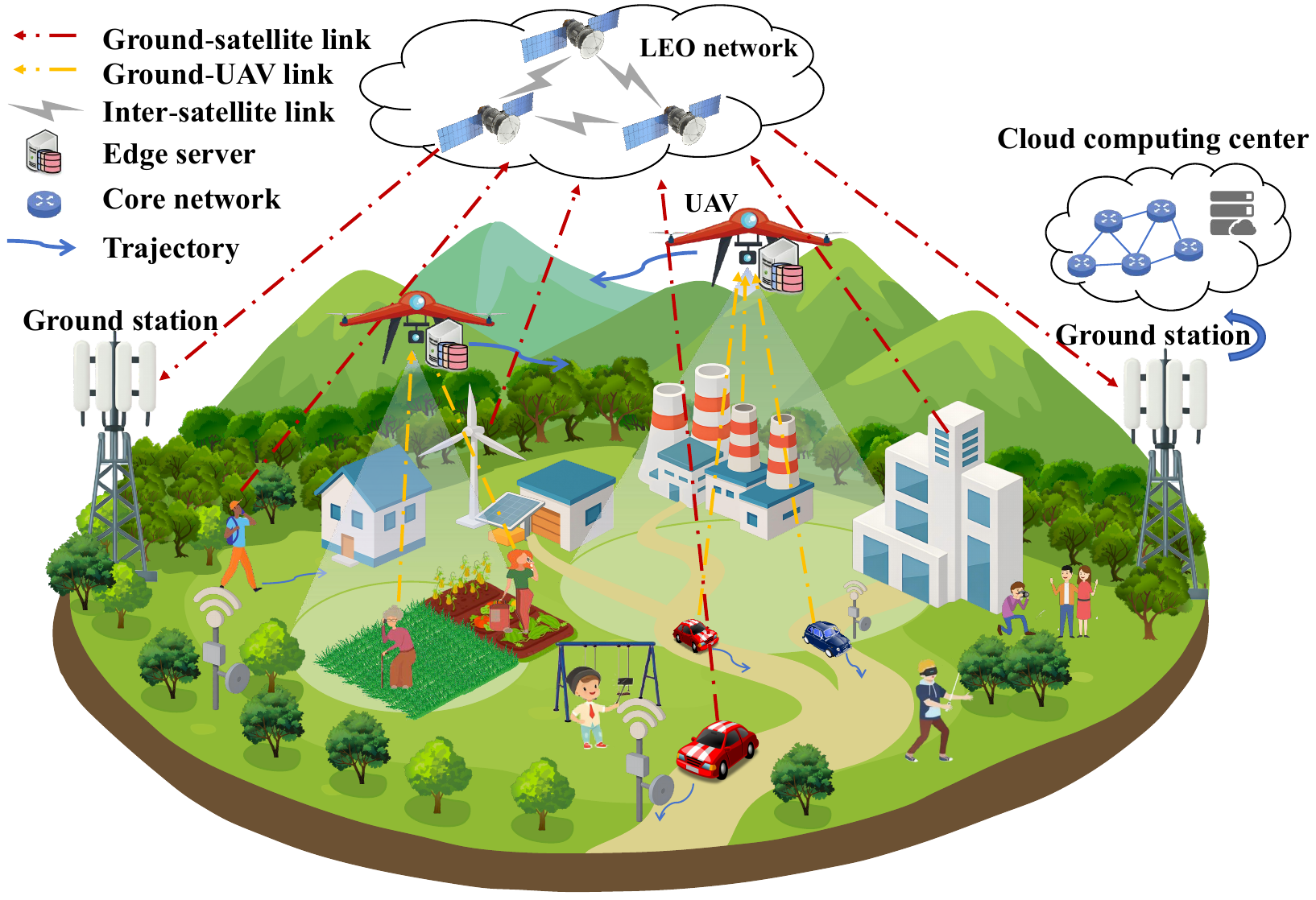}
	\caption{The architecture of the proposed SAGIN-MEC system in a remote area. In this hierarchical system, a set of UAVs in the air layer are deployed as mobile MEC servers to provide low-latency computing services for ground UDs. Meanwhile, LEO satellites in the space layer act as relays to forward tasks from UDs to a remote GS for cloud computing, thereby ensuring wide-area service coverage.}
	\label{fig_systemModel}
\end{figure}

%
%

\subsection{System Model}
\label{sec_system_model}

\subsubsection{System Overview}
\label{sec:system_overview}

\par As shown in Fig. \ref{fig_systemModel}, we consider an SAGIN-MEC system in the remote area where terrestrial infrastructures are unavailable~\cite{Gao2024}. Moreover, the proposed SAGIN-MEC system is structured by a hierarchical architecture, which consists of a ground layer, an air layer, and a space layer~\cite{cheng2019}. Specifically, \textit{at the ground layer}, a set of UDs $\mathcal{I}=\{1, \ldots, i, \ldots, I\}$ are randomly distributed in the target area, periodically generating various IoT tasks such as agricultural data collection and health monitoring. A GS $g$ is located far away from the considered area to provide cloud computing services. However, due to the lack of terrestrial infrastructures, the UDs can only establish connections with the GS through the space layer. \textit{At the air layer}, a set of UAVs $\mathcal{U}=\{1, \ldots, u, \ldots, U\}$ equipped with MEC servers offer proximate and flexible computing services to UDs within their coverage areas. \textit{At the space layer}, a set of LEO satellites, denoted by $\mathcal{N}=\{1, \ldots, n, \ldots, N\}$, is deployed to facilitate task forwarding from the UDs to the cloud. In other words, the satellites function as relays to offload tasks from UDs to the GS via inter-satellite links (ISLs). Moreover, we consider the SAGIN-MEC system operates in a discrete-time and quasi-static manner. Specifically, we discretize the system time into a set of $T$ time slots, denoted by $\mathcal{T}=\{1, \ldots, t, \ldots, T\}$, where each slot has an identical duration of $\tau$. Note that $\tau$ is sufficiently small so that the network topology remains constant within each time slot while varying across different time slots~\cite{Zhang2021}.

%
%
\subsubsection{Basic Models}
	
\par The basic models of the system are given as follows.

\textit{\textbf{UD Mobility Model.}} The horizontal coordinate of each UD $i\in\mathcal{I}$ is denoted as $\mathbf{q}_{i}(t)$. The temporal-dependent movement is captured via the Gauss-Markov model, similar to~\cite{Batabyal2015}. The location of each UD $i$ evolves as
\begin{equation}
	\label{eq_UD_position}
    \begin{aligned}
	\mathbf{q}_{i}(t+1)=\mathbf{q}_{i}(t) + \mathbf{v}_i(t) \tau,
    \end{aligned}
\end{equation}
\noindent where $\mathbf{v}_i(t)$ is the velocity vector derived from the Gauss-Markov model.

\par \textit{\textbf{UAV Mobility Model.}} We employ a three-dimensional Cartesian coordinate for each UAV $u\in\mathcal{U}$, i.e., $\mathbf{q}_{u}(t)=[q^x_u(t),q^y_u(t),H_u]^T$, where the UAV flies at a fixed altitude $H_u$~\cite{Wang2022a}. Therefore, the coordinate of UAV $u$ evolves as
\begin{subequations}
    \label{eq_AAV_position}
    \begin{alignat}{2}
        &x_u(t+1) = x_u(t) + \tau v_u(t) \cos \left(\theta_u(t)\right), \, \forall t \in \mathcal{T},\\
        &y_u(t+1)= y_u(t) + \tau v_u(t) \sin \left(\theta_u(t)\right), \, \forall t \in \mathcal{T},
    \end{alignat}
\end{subequations}

\noindent where $\theta_u(t) \in (-\pi, \pi]$ represents the instantaneous direction angle of UAV $u$, $v_u(t) \in [0,v_U^{\text{max}}]$ denotes the instantaneous velocity of UAV $u$, and $v_U^{\text{max}}$ is the maximum velocity constraint for the UAVs. 

\par Furthermore, the position of each UAV satisfies the physical constraints as follows:
\begin{subequations}
	\label{eq_AAV_mob_constraint}
	\begin{alignat}{2}
		& 0 \leq x_{u}(t) \leq x^{\max}, \quad &&\forall u \in \mathcal{U},\ t \in \mathcal{T}, \label{eq_AAV_mob_x}\\
		& 0 \leq y_u(t) \leq y^{\max}, \quad &&\forall u \in \mathcal{U},\ t \in \mathcal{T}, \label{eq_AAV_mob_y}\\
		&||\mathbf{q}_{u}(t)-\mathbf{q}_{u^{\prime}}(t)||\geq d_{\text{U}}^{\text{safe}}, &&\forall u, u^{\prime} \in \mathcal{U},\ u\neq u^{\prime},\ t \in \mathcal{T}. \label{eq_AAV_safe}
	\end{alignat}
\end{subequations}

\noindent Constraints~\eqref{eq_AAV_mob_x} and \eqref{eq_AAV_mob_y} ensure that each UAV flies within the boundaries of the considered area. Constraint \eqref{eq_AAV_safe} enforces a safety distance between UAVs to prevent collisions.

\begin{remark}
We adopt a two-dimensional (2D) UAV mobility model with a fixed altitude, which is a widely used simplification to focus on horizontal trajectory optimization. However, our proposed MADDPG-COCG algorithm can be extended to a three-dimensional (3D) model by incorporating the altitude of UAV as an additional continuous action~\cite{Sharma2019}. This extension would require more complex channel and energy consumption models, which will be systematically investigated in our future work.
\end{remark}

\par \textit{\textbf{UD Model.}} Each UD $i\in\mathcal{I}$ is characterized by a tuple $<\mathbf{q}_{i}(t), f_i^{\text{loc}}, E_i^{\max}>$, wherein $f_i^{\text{local}}$ denotes the computing resources of UD $i$ (in cycles/s) and $E_i^\text{max}$ denotes the energy constraint of UD $i$. Additionally, we consider that each UD generates a task per time slot~\cite{Wang2022}. The task generated by UD $i$ in time slot $t$ is defined by $\vartheta_i(t)=<\varpi_i(t), \eta_i(t), T_{i}^{\max}(t)>$, where $\varpi_i(t)$ is the task size (in bits), $\eta_i(t)$ represents the computation density of the task, and $T_{i}^{\max}(t)$ denotes the maximum tolerable delay of the task. 

\par \textit{\textbf{UAV Model.}} Each UAV $u\in\mathcal{U}$ is characterized by $<\mathbf{q}_{u}(t), f_u^{\max}, E_u^{\max}>$, where $f_u^{\max}$ denotes the computing resources (in cycles/s), and $E_u^\text{max}$ is the energy constraint of UAV $u$.

\par \textit{\textbf{Satellite Model.}} Each satellite $n$ is characterized by $<\mathbf{q}_n(t),H_n>$,  where $\mathbf{q}_{n}(t)$ and $H_n$ represent the horizontal coordinate in time slot $t$ and the altitude of satellite $n$, respectively.

%
%
\subsection{Communication Model}
\label{sec_communicationModel}

\par Each UD establishes communication either with a UAV through a ground-to-air link or with a satellite via a ground-to-space link. The corresponding communication models are presented in detail in the following subsection.

\subsubsection{UD-UAV Communication} 

\par We consider that the communication link between UD and UAV operates at C-band \cite{kassem2017analysis}. Furthermore, UAVs utilize orthogonal frequency division multiple access to concurrently serve multiple UDs~\cite{ji2023}. Hence, the transmission rate from UD $i$ to UAV $u$ is calculated as
\begin{equation}
    \label{eq_trans_rate_AAV}
    \begin{aligned}
        R_{i,u}(t)=B_{i,u}(t)\log_{2}(1+p_{i,u}(t)h_{i,u}(t)/\sigma^2),
    \end{aligned}
\end{equation}

\noindent where $B_{i,u}(t)$ denotes the bandwidth allocated to UD $i$ from UAV $u$, determined by equally dividing the total bandwidth $B_u^\text{total}$ of UAV $u$ among all UDs it serves, $p_{i,u}(t)$ represents the transmit power from UD $i$ to UAV $u$, $\sigma^2$ is the background noise power. Additionally, $h_{i,u}(t)$ means instantaneous channel gain between UD $i$ and UAV $u$, which is given as
\begin{equation}
    \label{eq_trans_rate_AAV}
    \begin{aligned}
        h_{i,u}(t)=10^{-L_{i,u}(t)/10},
    \end{aligned}
\end{equation}

\noindent where $L_{i,u}(t)$ denotes the path loss between UD $i$ and UAV $u$. Considering that the UD-UAV communication often experiences occasional obstruction cuased by obstacles, we calculate $L_{i,u}(t)$ by employing the probabilistic line-of-sight (LoS) channel model~\cite{Shi2018}, as follows:
\begin{equation}
    \begin{aligned}
    \label{eq_pathloss_AAV}
    L_{i,u}(t) &= 20\log(4 \pi \epsilon d_{i,u}(t)/c) \\
    &+ P_{i,u}(t) \rho^{\text{L}} + (1-P_{i,u}(t)) \rho^{\text{N}},
    \end{aligned}
\end{equation}

\noindent where $\rho^{\text{L}}$ and $\rho^{\text{N}}$ represent the additive path loss for LoS and none-line-of-sight (NLoS) links, respectively~\cite{AlHourani2014}. Additionally, $d_{i,u}(t)=(\|\mathbf{q}_i(t)-\mathbf{q}_u(t)\|^2+H_u^2)^{1/2}$ indicates the distance between UD $i$ and UAV $u$, $\epsilon$ denotes the carrier frequency of the C-band, and $c$ is the speed of light. Besides, $P_{i,u}(t)$ means the probability of the LoS link, i.e.,
\begin{equation}
    \label{eq_P_LoS_AAV}
    \begin{aligned}
        P_{i,u}(t)=\frac{1}{1+\varepsilon_1\exp\big(-\varepsilon_2(\frac{180\theta_{i,u}(t)}{\pi}-\varepsilon_1)\big)},
    \end{aligned}
\end{equation}

\noindent where $\varepsilon_1$ and $\varepsilon_2$ are the environment-dependent variables~\cite{Yaliniz2016}. Additionally, $\theta_{i,u}(t)=\arcsin{(\frac{H_u}{\|\mathbf{q}_i(t)-\mathbf{q}_u(t)\|})}$ denotes the elevation angle between UD $i$ and UAV $u$.

\begin{remark}
We focus on tractable single-cell resource management by neglecting inter-cell interference (ICI) in this work. However, as ICI is critical in dense multi-UAV scenarios, our MADDPG-COCG algorithm can be readily extended by including interference information in the states of agents, thus enabling cooperative interference management. This extension will be systematically considered in our future work.
\end{remark}

\subsubsection{UD-satellite Communication}

\par Similar to~\cite{Zhou2021}, we consider that the UD-satellite communication operates at Ka-band. Then, the transmission rate from UD $i$ and satellite $n$ is given as 
\begin{equation}
    \label{eq_trans_rate_sate}
    \begin{aligned}
        R_{i,n}(t)=B_{i,n}(t)\log_2\big(1+p_{i,n}(t)h_{i,n}(t)/\sigma_{\text{US}}^2\big),
    \end{aligned}
\end{equation}

\noindent where $B_{i,n}(t)$ represents the bandwidth between UD $i$ and satellite $n$, and $p_{i,n}(t)$ denotes the transmit power from UD $i$ to satellite $n$. Furthermore, $h_{i,n}(t)$ is the instantaneous channel gain between UD $i$ and satellite $n$, which is expressed as
\begin{equation}
    \label{eq_trans_rate_AAV}
    \begin{aligned}
        h_{i,n}(t)=10^{-\frac{L_{i,n}(t)}{10}},
    \end{aligned}
\end{equation}

\noindent where $L_{i,n}(t)$ denotes the path loss between UD $i$ and satellite $n$. Considering that the UD-satellite communication is sensitive to the environmental factors such as rainfall, $L_{i,n}(t)$ is calculated by integrating the rain attenuation and the path loss fading~\cite{Shi2018}, as follows:
\begin{equation}
    \label{eq_trans_loss_sate}
    \begin{aligned}
        L_{i,n}(t)= 20\log(4 \pi \epsilon^{\text{Ka}} d_{i,n}(t)/c)+\varrho,
    \end{aligned}
\end{equation}

\noindent where $d_{i,n}(t)=(\|\mathbf{q}_i(t)-\mathbf{q}_n(t)\|^2+H_n^2)^{1/2}$ denotes the distance between UD $i$ and satellite $n$, $\epsilon^{\text{Ka}}$ means the carrier frequency of the Ka-band. Additionally, $\varrho$ indicates the rain attenuation, which is modeled as a Weibull-based stochastic process~\cite{Kanellopoulos2014}.

%
%

\subsection{Computation Model}
\label{sec_DelayModel}

\par For task $\vartheta_i(t)$ generated by UD $i$ at time slot $t$, UD $i$ can offload a portion of the task to a UAV for edge computing or to a satellite for cloud computing, which depends on the UD association decision. More specifically, we define a variable $\chi_i(t)\in \mathcal{U} \cup \mathcal{N}$ to represent the association decision, where $\chi_i(t)=u$ indicates UD $i$ is associated with UAV $u$, and $\chi_i(t)=n$ represents that UD $i$ is associated with satellite $n$. In other words, $\chi_i(t)=u$ and $\chi_i(t)=n$ mean that task $\vartheta_i(t)$ is offloaded to UAV $u$ and to satellite $n$, respectively. Additionally, we introduce a variable $\lambda_i(t)\in[0,1]$ to denote the task offloading ratio.

\subsubsection{Local Computing Model}
\par If a portion of task $\vartheta_i(t)$ is processed by UD $i$ locally, the task completion delay and energy consumption can be calculated as follows.

\par \textbf{Task Completion Delay.} The task completion delay for local computing is given as 
\begin{equation}
    \label{eq_time_local}
    \begin{aligned}
        T_{i}^{\text{loc}}(t)=\eta_i(t)(1-\lambda_{i}(t))\varpi_i(t)/f_i^{\text{loc}},
    \end{aligned}
\end{equation}

\noindent where $(1-\lambda_{i}(t))$ denotes the portion of task $\vartheta_i(t)$ that is processed locally by UD $i$.

\par \textbf{UD Energy Consumption.} The corresponding computing energy consumption of UD $i$ is calculated as
\begin{equation}
    \label{eq_energy_local}
    \begin{aligned}
        E_{i}^{\text{loc}}(t)=\eta_i(t)\varsigma_i({f_i^{\text{loc}}})^2\left(1-\lambda_{i}(t)\right)\varpi_i(t),
    \end{aligned}
\end{equation}

\noindent where $\varsigma_i\geq0$ represents the effective capacitance of UD $i$, which depends on the CPU chip architecture~\cite{Miettinen2010}.

\subsubsection{Edge Computing Model}

\par If task $\vartheta_i(t)$ is processed by UAV $u$ ($\chi_{i}(t)=u$) for edge computing, the UAV allocates computing resources to execute the task. Note that we omit the result feedback delay, as the processing results of most mobile applications are typically much smaller than the input data~\cite{zhang2024energy}. 

\par \textbf{Task Completion Delay.} The task completion delay for UAV offloading mainly consists of transmission delay and computation delay, which is given as
\begin{equation}
    \label{eq_trans_delay_AAV}
    \begin{aligned}
        T_{i,u}(t)=\underbrace{\lambda_{i}(t)\varpi_i(t)/R_{i,u}(t)}_{\text{Transmission delay}}+\underbrace{\eta_i(t)\lambda_{i}(t)\varpi_i(t)/f_{u,i}(t)}_{\text{Computation delay}},
    \end{aligned}
\end{equation}

\noindent where $f_{u,i}(t)$ denotes the computing resources allocated by UAV $u$ for task $\vartheta_i(t)$.

\par \textbf{UD Energy Consumption.} Task transmission is the primary source of energy consumption for each UD, expressed as
\begin{equation}
    \label{eq_time_local}
    \begin{aligned}
        E_{i,u}(t)=p_{i,u}(t)\lambda_{i}(t)\varpi_i(t)/R_{i,u}(t).
    \end{aligned}
\end{equation}

\par \textbf{UAV Energy Consumption.} The corresponding energy consumption of UAV includes the computation and the flight energy, which is given as
\begin{equation}
    \label{eq_energy_AAV}
    \begin{aligned}
        E_{u}(t)=\sum_{i\in\mathcal{I}}\big(\mathbb{I}_{\{\chi_i(t)=u\}}E_{i,u}^{\text{com}}(t)\big)+E_u^{\text{pro}}(t),
    \end{aligned}
\end{equation}

\noindent where $\mathbb{I}_{\{\chi\}}$ denotes an indicator function such that $\mathbb{I}_{\{\chi\}}=1$ if $\chi$ holds, and $\mathbb{I}_{\{\chi\}}=0$ otherwise. Moreover, $E_{i,u}^{\text{com}}(t)$ represents the energy consumption of UAV $u$ to process task $\vartheta_i(t)$, which is given as
\begin{equation}
    \label{eq_comp_energy_AAV}
    \begin{aligned}
        E_{i,u}^{\text{com}}(t)=\eta_i(t)\lambda_{i}(t)\varpi_i(t)\mu_u,
    \end{aligned}
\end{equation}

\noindent where $\mu_u$ represents the energy expended by UAV $u$ per CPU cycle~\cite{Jiang2023}. Furthermore, $E_u^{\text{pro}}(t)$ indicates the UAV propulsion energy consumption, which is calculated as~\cite{Zeng2019}
\begin{sequation}
    \begin{aligned}
	\label{eq_AAV_flight}		
            E_{u}^\text{pro}(t)&=\Big(\underbrace{\delta_1\big(1+3 (v_{u}(t))^2/{v_u^{\text{tip}}}^2\big)}_{\text {Blade profile power}}+\underbrace{\delta_4 (v_{u}(t))^3}_{\text {Parasite power}}\\&+\underbrace{\delta_2 \sqrt{\sqrt{\delta_3+(v_{u}(t))^4/4}}-(v_{u}(t))^2/2}_{\text{Induced power}}\Big)\tau,
    \end{aligned}
\end{sequation}

\noindent where $v_u(t)$ denotes the velocity of UAV $u$ at time slot $t$, and $v_u^{\text{tip}}$ represents the rotor tip speed of UAV $u$. Additionally, $\delta_1$, $\delta_2$, $\delta_3$ and $\delta_4$ are constants that depend on the aerodynamic parameters of the UAV~\cite{Yang2022}.

\subsubsection{Cloud Computing Model}

\par If a portion of task $\vartheta_i(t)$ is processed by the cloud, UD $i$  first offloads this portion of task to satellite $n$ with which it is associated ($\chi_{i}(t)=n$), and the satellite then forwards the task to the GS for cloud computing. More specifically, if there is no GS within the coverage area of satellite $n$, satellite $n$ will route the task through ISLs along the shortest path to the satellite $n^{\prime}$ that can establish a connection to the GS. Note that we omit the forwarding energy consumption of satellites, as it is significantly smaller compared to the energy consumption of UAVs, and satellites are typically equipped with more abundant and stable energy resources \cite{cheng2019}. We also ignore the computation delay and energy consumption of clouding computing, as the cloud possesses powerful computing capabilities and abundant energy resources \cite{dinh2013survey}.

\par \textbf{Task Completion Delay.} For the portion of task $\vartheta_i(t)$ processed by cloud, the task completion delay comprises the uploading delay from the UD to the satellite, forwarding delay between satellites, the downloading delay from the satellite to the GS, and the propagation delay, which is given as
\begin{equation}
    \label{eq_trans_delay_sate}
    \begin{aligned}
        &T_{i,n}(t)=\underbrace{\lambda_{i}(t)\varpi_i(t)/R_{i,n}(t)}_{\text{Uploading delay}} + \underbrace{j_{n,n^{\prime}}^\text{hop}(t)\lambda_{i}(t)\varpi_i(t)/R_\text{ISL}}_{\text{Forwarding delay}}\\
        &+ \underbrace{\lambda_{i}(t)\varpi_i(t)/R_\text{SG}}_{\text{Downloading delay}} +\underbrace{2\left(d_{i,n}(t)+d_{n,n^{\prime}}(t)+d_{n^{\prime},g}(t)\right)/c}_{\text{Propagation delay}},
    \end{aligned}
\end{equation}

\noindent where $R_\text{ISL}$ and $R_\text{SG}$ represent the transmission rate over the ISLs and the downlink transmission rate from a satellite to the GS, respectively, both of which are fixed values~\cite{cheng2019}. Additionally, $j_{i,n}^\text{hop}$ denotes the number of hops between satellite $n$ and $n^{\prime}$. Furthermore, $d_{n,n^{\prime}}(t)$ is the total distance of the ISL chain from satellite $n$ and satellite $n^{\prime}$. Besides, $d_{n^{\prime},g}(t)$ is the distance between satellite $n^{\prime}$ and the GS\footnote{Note that if satellite $n$ is directly connected to the GS, we can know $n=n^{\prime}$, $j_{n,n^{\prime}}^\text{hop}=0$, and $d_{n,n^{\prime}}(t)=0$.}.

\par \textbf{UD Energy Consumption.} Similarly, the energy consumption of UD $i$ is primarily due to task transmission, which is calculated as 
\begin{equation}
    \label{eq_trans_energy_sate}
    \begin{aligned}
        E_{i,n}(t)=p_i(t)\lambda_{i}(t)\varpi_i(t)/R_{i,n}(t).
    \end{aligned}
\end{equation}

%
%

\subsection{UD Cost Model}
\label{sec_CostModel}

\par Due to the delay sensitivity of tasks and the limited battery capacity of UDs, energy consumption and task completion delay are key performance metrics for assessing the quality of experience for UDs. Therefore, we present the UD cost model by integrating the task completion and UD energy consumption. Specifically, since we consider partial offloading, the completion delay for task $\vartheta_i(t)$ is calculated as the maximum between the local computing delay and the delay for edge (or cloud) computing, i.e.,
\begin{equation}
    \label{eq_total_delay}
    \begin{aligned}
        T_i(t)=\max\big({T_{i}^\text{loc}(t), \sum_{k\in\mathcal{U}\cup\mathcal{N}}\mathbb{I}_{\{\chi_i(t)=k\}}T_{i,k}(t)}\big),   
    \end{aligned}
\end{equation}

\noindent where $k$ represents a UAV or satellite. Then, the total energy consumed by UD $i$ for task $\vartheta_i(t)$ comprises the energy for local computation and task transmission, given by
\begin{equation}
    \label{eq_energy_UD}
    \begin{aligned}
        E_i(t)=E_i^\text{loc}(t)+\sum_{k\in\mathcal{U}\cup\mathcal{N}}\mathbb{I}_{\{\chi_i(t)=k\}}E_{i,k}(t).
    \end{aligned}
\end{equation}

\noindent Based on \eqref{eq_total_delay} and  \eqref{eq_energy_UD}, the total cost of the UDs at time slot $t$ can be calculated as
\begin{equation}
    \label{eq_cost_task}
    \begin{aligned}
        C(t)&=w^T\sum_{i\in \mathcal{I}}T_i(t) + w^E\sum_{i\in\mathcal{I}}E_i(t),
    \end{aligned}
\end{equation}

\noindent where $w^T$ and $w^E$ represent the weight coefficients for task completion delay and UD energy consumption, respectively.

%
%
\section{Problem Formulation and Analysis}
\label{sec_problemFormulation}

\subsection{Problem Formulation}

\par The objective of this work is to minimize the total cost of UDs by jointly optimizing the task offloading ratio $\mathbf{\Lambda}=\{\lambda_{i}(t)\}_{i\in \mathcal{I}, t \in \mathcal{T}}$, UAV trajectory planning $\mathbf{Q}=\{\mathbf{q}_{u}(t)\}_{u\in \mathcal{U}, t \in \mathcal{T}}$, computing resource allocation $\mathbf{F}=\{f_{u,i}(t)\}_{u\in \mathcal{U}, i \in \mathcal{I}, t \in \mathcal{T}}$, and UD association $\mathbf{X}=\{\chi_{i}(t)\}_{i\in \mathcal{I}, t \in \mathcal{T}}$. Therefore, UCMOP can be formulated as
\begin{subequations}
	\label{eq_problem}
	\begin{alignat}{2}
		\mathbf{P}: \quad &\min_{\mathbf{\Lambda},\mathbf{Q},\mathbf{F}, \mathbf{X}}  \sum_{t=1}^T C(t), \label{utility}\\
		\text{s.t.} \quad & \lambda_{i}(t) \in[0,1], \ \forall i\in \mathcal{I}, 
        \ t\in \mathcal{T},\label{pro_c1}\\
            & \chi_{i}(t)\in\mathcal{U}\cup \mathcal{N}, \ \forall i\in \mathcal{I}, \ t\in 
        \mathcal{T},\label{pro_c2}\\
            & T_i(t)\leq T_i^\text{max}(t), \ \forall i\in \mathcal{I}, \ t\in 
        \mathcal{T},\label{pro_c3}\\
            & 0 \leq f_{u,i}(t) \leq f_u^\text{max}, \ \forall u\in \mathcal{U}, 
        \ i\in \mathcal{I}, \ t\in \mathcal{T},\label{pro_c4}\\
            & \sum_{i \in \mathcal{I}}\mathbb{I}_{\{\chi_i(t)=u\}}f_{u,i}(t) \leq f_u^\text{max}, \ 
        \forall  i \in \mathcal{I}, u \in \mathcal{U},\ t\in \mathcal{T},\label{pro_c5}\\
            & \sum_{t \in \mathcal{T}}E_u(t) \leq E_u^\text{max}, \ \forall u \in 
        \mathcal{U},\ t\in \mathcal{T},\label{pro_c6}\\
            & \sum_{t \in \mathcal{T}}E_i(t) \leq E_i^\text{max}, \ \forall i \in 
        \mathcal{I},\ t\in \mathcal{T},\label{pro_c7}\\ 
            & \eqref{eq_AAV_mob_x} \sim \eqref{eq_AAV_safe},\label{pro_c8}
	\end{alignat}
\end{subequations}

\noindent where constraint \eqref{pro_c1} limits the value of task offloading ratio. Constraint \eqref{pro_c2} indicates that each UD can be associated with one UAV or satellite. Moreover, constraint \eqref{pro_c3} ensures that each task completes within its maximum tolerable delay. Furthermore, constraints \eqref{pro_c4} and \eqref{pro_c5} constrain the computing resource allocation of each UAV. Constraints \eqref{pro_c6} and \eqref{pro_c7} restrict the energy consumption. Besides, constraint \eqref{pro_c8} poses limits on the mobility of each UAV.

\subsection{Problem Analysis}
\label{sec_problem_analy}

\par Directly solving problem $\textbf{P}$ is challenging for several reasons:
    \begin{itemize}
    \item \textit{NP-hard and Non-convex Problem.} Problem $\mathbf{P}$ includes both discrete variables (i.e., UD association $\mathbf{X}$), continuous variables (i.e., task offloading ratio $\Lambda$, UAV trajectory planning $\mathbf{Q}$, and computing resource allocation $\mathbf{F}$), along with a non-linear objective function. Consequently, problem $\mathbf{P}$ is an mixed-integer nonlinear programming (MINLP) problem, which is both non-convex and NP hard~\cite{boyd2004convex}, \cite{sun2024multi_objective}.
    
    \item \textit{Dynamic and Partially Observable Environment.} The SAGIN-MEC system operates in a dynamic environment characterized by random task arrivals, real-time UD mobility, fluctuating MEC server capacities, and time-varying wireless channels \cite{sun2025tjcct}. Additionally, the lack of prior knowledge about environmental changes further complicates the prediction of future states. This combination of dynamic features and partial observability introduces significant challenges in designing adaptive and real-time decision-making frameworks.
    
    \item \textit{Heterogeneous Decision-making Nodes.} The SAGIN-MEC system comprises heterogeneous nodes, namely UDs and UAVs, each with distinct decisions and resource constraints. Specifically, UDs determine their task offloading ratios to optimize delay and energy consumption, while UAVs plan their trajectories to meet the UD requirements while adhering to their own energy constraints. However, since the decisions of different nodes are interdependent and collectively influence the UD costs, coordinating decisions across these heterogeneous nodes poses a significant challenge.

    \item \textit{Varying-dimensional and Hybrid Decision Variables.} The number of tasks offloaded to UAVs varies dynamically across time slots, thus resulting in fluctuations in the dimensionality of resource allocation decisions. Moreover, the decision variables span hybrid discrete-continuous domains, which complicates the optimization process \cite{zhang2025multi-objective-diffusion}. Additionally, this variability adds another layer of complexity, making it challenging to ensure convergence and computational efficiency.
\end{itemize}

%
%

\begin{figure*}[t] 
	\centering
        \setlength{\abovecaptionskip}{2pt}%
	\setlength{\belowcaptionskip}{2pt}%
	\includegraphics[width =7in]{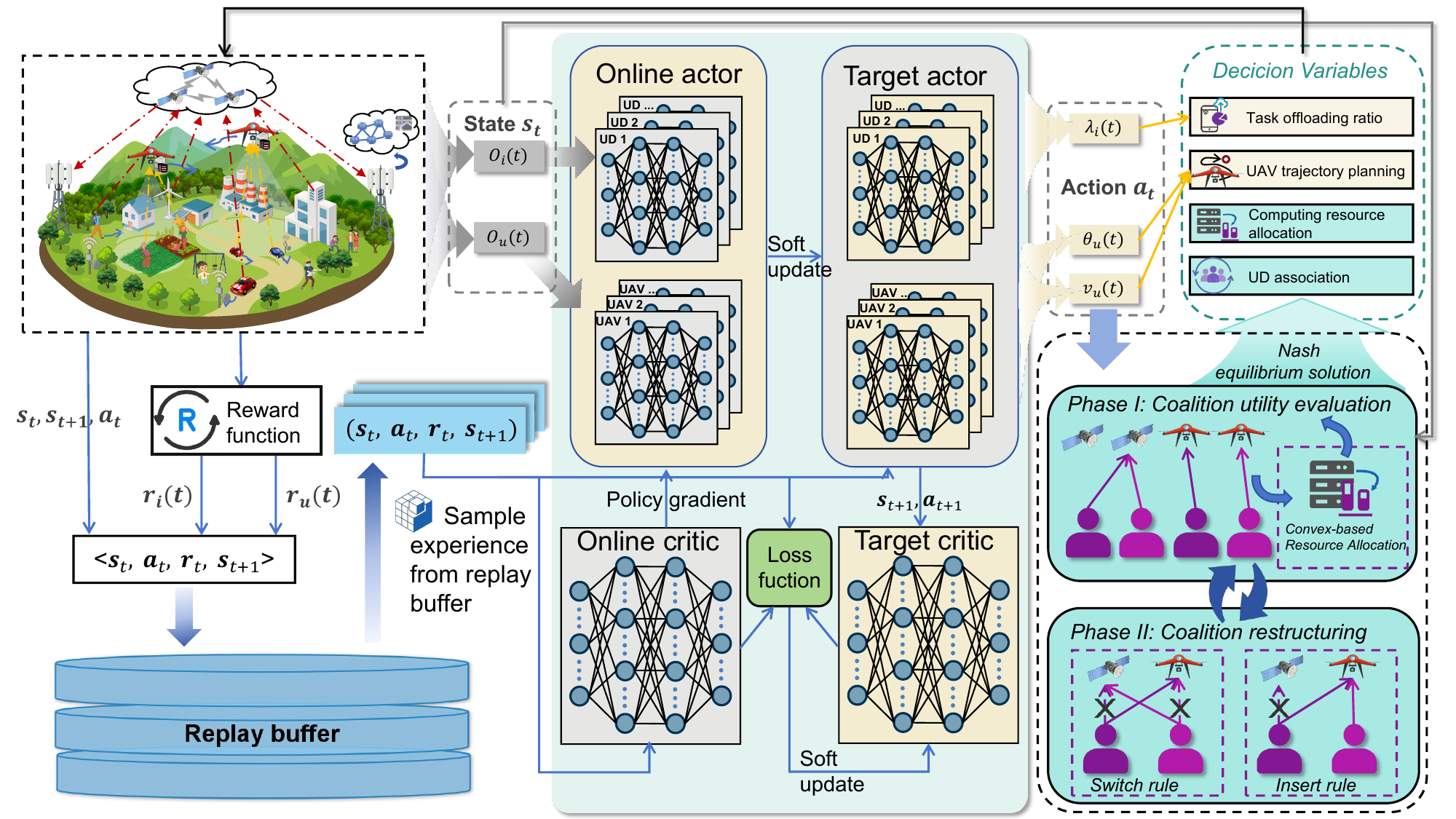}
    \caption{The proposed MADDPG-COCG algorithm addresses hybrid and varying-dimensional action spaces by decoupling the decision variables. The MADDPG component learns continuous decisions, namely, the task offloading ratio and UAV trajectory planning. To enhance learning stability, the COCG module deterministically handles the remaining decisions: deriving a closed-form solution for computing resource allocation via convex optimization , and obtaining a stable UD association structure via a coalitional game}  
	\label{fig_MADDPG-COCG}
     \vspace{-2.0em}
\end{figure*}

\section{The Proposed MADDPG-COCG}
\label{sec_algorithm}

    \par In this section, we first present the motivations for proposing MADDPG-COCG. Then, we introduce MADDPG-COCG in detail. The framework of the proposed MADDPG-COCG is illustrated in Fig. \ref{fig_MADDPG-COCG}.

\subsection{Motivations} 
\label{Motivation}

    \par The motivations for proposing MADDPG-COCG are presented as follows.
    
    \begin{itemize}
    \item \textit{Leveraging DRL for Long-term Optimization in Dynamic Environments.} Traditional optimization methods often focus on short-term objectives, which restricts their ability to achieve cumulative performance improvements over extended periods. In contrast, DRL-based methods are inherently designed to optimize long-term rewards by continuously learning from interactions with the environment. This capability makes DRL well-suited for adapting to the dynamic nature of SAGIN-MEC systems and optimizing cumulative UD costs over multiple time slots.

    \item \textit{Introducing MADRL for Coordinated Decision Making.} Despite the advantages of DRL, traditional DRL algorithms face inherent limitations when applied to heterogeneous decision-making nodes with interdependent decisions. Traditional DRL algorithms struggle to effectively capture the interdependencies of the decision variables and fail to learn decentralized and coordinated actions. This motivates the adoption of MADRL algorithm since it enables both UAVs and UDs to act as agents and learn optimal policies collaboratively through decentralized execution and centralized training. Therefore, MADRL can efficiently capture the interactions and dependencies among agents, thus ensuring scalable and adaptive decision-making in dynamic and partially observable environments.

    \item \textit{Proposing the MADDPG-COCG to enhance the MADRL.} MADRL algorithms often face significant challenges in achieving convergence when dealing with the varying-dimensional and hybrid decision variables. Therefore, to enhance the learning efficiency and stability of MADRL with such complex decision variables, we propose an MADDPG-COCG algorithm. Specifically, we employ MADDPG to ensure stable learning of the continuous decisions of task offloading ratio and UAV trajectory planning for heterogeneous nodes in the partially observable SAGIN-MEC system. Moreover, we design a COCG method to enhance the MADDPG by handling the varying-denominational and hybrid decision variables. On the one hand, we leverage convex optimization to derive the closed-form solution for computing resource allocation by using its ability to handle varying dimensions of resource allocation decisions. On the other hand, we introduce a coalitional game to enable distributed and mutually beneficial association decisions~\cite{Chen2021}.
\end{itemize}

\subsection{Partially Observable Markov Game (POMG) for UCMOP}
\label{sec_POMG}

\par To address the limited observability and interactions among UAVs and UDs, the formulated UCMOP can be modeled as a POMG. The POMG is defined by a tuple $(\mathcal{M}, \mathcal{S}, \mathcal{O}, \mathcal{A}, \mathcal{P}, \mathcal{R})$, where $\mathcal{M}=\mathcal{I}\cup\mathcal{U}$ represents the set of agents that compromises both UD agents and UAV agents, $\mathcal{S}$ denotes the set of global environment states, $\mathcal{O}=\mathcal{O}^\text{UD} \cup \mathcal{O}^\text{UAV}$is the set of observations, $\mathcal{A}=\mathcal{A}^\text{UD} \cup \mathcal{A}^\text{UAV}$ represents the joint action space, $\mathcal{P}$ denotes the state transition function, and $\mathcal{R}=\mathcal{R}^\text{UD} \cup \mathcal{R}^\text{UAV}$ is the joint reward function.

\par Under the framework of the POMG, the environment state in time slot $t$ is denoted as $s(t)\in\mathcal{S}$. Subsequently, each agent $m$ obtains its observation $o_{m}(t)\in\mathcal{O}$ through its observation function and selects an action $a_{m}(t)\in\mathcal{A}$ based on its policy. Then, by executing the joint actions $a(t)$ of all agents, the environment provides each agent a reward $r_{m}(t)\in\mathcal{R}$ and transitions to the next state $s(t+1)$ according to the state transition probability function $\mathcal{P}(s(t+1) \mid s(t),a(t))$. Accordingly, the key elements in the POMG are detailed as follows.

\subsubsection{State Space}
\label{State}
\par The environment state can be defined as $\mathcal{S}=\{\mathbf{q}_m(t),E_m(t),\varpi_i(t),\eta_i(t)\mid \forall m\in \mathcal{M}, i\in \mathcal{I},t\in\mathcal{T}\}$, where $\mathbf{q}_m(t)$ represents the position of agent $m$, $E_m(t)$ denotes the remaining energy of agent $m$, $\varpi_i(t)$ indicates the size of tasks generated by UD $i$, and $\eta_i(t)$ is the computation density of the task generated by UD $i$.

\subsubsection{Observation Space}
\label{Observation}
\par Each agent can only observe a portion of the global environment state due to the limited radio range. Therefore, the observations of UDs and UAVs are given as $\mathcal{O}^\text{UD}=\{o_i(t)\mid \forall i \in \mathcal{I}, t\in\mathcal{T}\}$ and $\mathcal{O}^\text{UAV}=\{o_u(t) \mid \forall u \in \mathcal{U}, t\in\mathcal{T}\}$, respectively, where $o_i(t)=\{\mathbf{q}_i(t), \varpi_i(t), \eta_i(t), E_i(t)\}$ and $o_u(t)=\{\mathbf{q}_m(t), E_u(t)\}$.

\subsubsection{Action Space}
\label{Action}
\par The joint action space $\mathcal{A}$ consists of the set of decisions for the formulated UCMOP. Specifically, the action space of UDs is denoted as $\mathcal{A}^\text{UD}=\{a_i(t)\mid\forall i \in \mathcal{I},t\in\mathcal{T}\}$, where $a_i(t)=\{\lambda_i(t),\chi_{i}(t)\}$. Moreover, the action space of a UAV is denoted as $\mathcal{A}^\text{UAV}=\{a_u(t)\mid\forall u \in \mathcal{U},t\in\mathcal{T}\}$, where $a_u(t)=\{\theta_u(t),v_u(t),f_{u,i}(t)\mid \forall i \in\mathcal{I},\mathbb{I}_{\{\chi_i(t)=u\}} = 1\}$.

\subsubsection{Reward Function}
\label{Reward}
\par The reward is a function of observations and actions, which evaluates the effectiveness of agent decisions. To balance individual and system-wide objectives~\cite{Wang2022b}, the reward is designed as follows.

\textbf{UD reward.} The reward for UDs is given as $\mathcal{R}^\text{UD}=\{r_i(t) \mid \forall i \in \mathcal{I},t\in\mathcal{T}\}$, where the reward of each UD is calculated as
\begin{equation}
    \label{eq_reward_UD}
    r_i(t)=w^S \left(-C(t)+ r^p(t)\right) -w^I r_i^I(t),
\end{equation}

\noindent where $w^S$ and $w^I$ represent the weighting factors for the system reward and the individual reward, respectively. Furthermore, $r_i^I(t)$ represents the individual reward of UD $i$, which is calculated as $r_i^I(t) = w^TT_i(t)+w^EE_i(t)$. Additionally, $r^p(t)$ denotes the sum of penalties for all tasks that exceed their deadlines, which can be calculated as
\begin{equation}
    r^p(t)=-\sum_{i \in \mathcal{I}}\xi_i^d(t)r_i^d(t),
\end{equation}

\noindent where $\xi_i^d(t)$ is a binary indicator that indicates whether a task exceeds its deadline, and $r_i^d(t)$ denotes the corresponding penalty.

\textbf{UAV reward.} The reward for UAVs can be represented as $\mathcal{R}^\text{UAV}=\{r_u(t)\mid\forall u \in \mathcal{U},t\in\mathcal{T}\}$, where
\begin{equation}
    \label{eq_reward_UAV}
    r_u(t)=w^S \left(-C(t)+ r^p(t)\right) +w^I r_u^U(t).
\end{equation}

\noindent $r_u(t)$ denotes the individual reward of UAV $u$ at time slot $t$, which can be calculated as 
\begin{equation}
    \label{eq_r_indi_UAV}
    r_u^U(t)=\sum_{i\in  \mathcal{I}}\mathbb{I}_{\{\chi_i(t)=u\}}r_i^I(t) - \xi_u^b(t)r_u^b(t) - \xi_u^c(t)r_u^c(t),
\end{equation}

\noindent where $\xi_u^b(t)$ and $\xi_u^c(t)$ are binary indicators that denote whether the UAV exceeds its boundary or experiences a collision, respectively. Additionally, $r_u^b(t)$ and $r_u^c(t)$ represent the penalty for boundary violations and collisions, respectively. It is worth noting that unlike the penalty for task deadline violations, the penalties for UAV boundary violations and collisions are more closely related to the decisions of UAV trajectory planning. Therefore, these penalties are included in the individual reward of the UAV, rather than in the system reward.

\subsubsection{POMG Analysis}
\label{POMG_analysis}
\par In the POMG model, the joint action space involves four types of decision variables, i.e., task offloading ratio, UAV trajectory planning, computing resource allocation, and UD association. However, although the MADRL excels at handling the complex decisions, it struggles with varying-dimensional and hybrid discrete-continuous decision variables since these complexities introduce significant convergence and stability issues, as mentioned in Section \ref{sec_problem_analy}. Therefore, to achieve a more efficient and stable learning process, we split the decision variables in the action space. Specifically, we employ the MADDPG to learn continuous and temporal decisions of task offloading ratio and UAV trajectory planning for heterogeneous nodes, which is detailed in Section \ref{sec_MADDPG}. Moreover, we design the COCG method to deterministically handle the varying-dimensional decision of computing resource allocation and discrete decision of UD association, which is detailed in Section \ref{sec_COCG}.

\subsection{MADDPG Algorithm}
\label{sec_MADDPG}

\par Based on the discussion in Section \ref{sec_POMG}, the MADRL algorithm learns the decisions of task offloading ratio $\Lambda$ and UAV trajectory planning $\mathcal{Q}$. Accordingly, the action space of the POMG can be reformulated as $\mathcal{A}=\mathcal{A}^\text{UD}\cup\mathcal{A}^\text{UAV}$, where $\mathcal{A}^\text{UD}=\{\lambda_i(t)\mid\forall i \in \mathcal{I},t\in\mathcal{T}\}$, and $\mathcal{A}^\text{UAV}=\{\theta_u(t),v_u(t)\mid\forall u \in \mathcal{U},t\in\mathcal{T},\mathbb{I}_{\{\chi_i(t)=u\}} = 1\}$. To effectively determine the continuous actions of task offloading ratio and UAV trajectory planning, the DDPG algorithm is particularly suitable due to its capacity to learn deterministic policies in continuous action spaces~\cite{zhang2025a}. By integrating DDPG into the MADRL algorithm, the resulting MADDPG algorithm enables both UDs and UAVs to optimize their continuous actions cooperatively by leveraging centralized training with decentralized execution. This makes MADDPG particularly suited for the joint optimization of task offloading ratio and UAV trajectory planning in partially observable SAGIN-MEC systems, thereby ensuring improved adaptability and enhanced user experience. Therefore, we adopt MADDPG to address the optimization of task offloading ratio and UAV trajectory planning.  

\subsubsection{Actor–Critic Framework}

\par The MADDPG has an actor-critic structure, which consists of $M$ critic networks and $M$ actor networks. Specifically, MADDPG utilizes a centralized training strategy, which allows the critic network to access global state information, including the observations and actions of all agents in the environment. This global perspective helps stabilize the learning process by considering the interdependencies and interactions between agents. Moreover, in the execution phase, each agent independently uses its own actor network, which only relies on its local observations. Such decentralized execution ensures that agents can operate efficiently in a distributed manner without requiring access to global information. The critic network and actor network of agent $m\in\mathcal{M}$ are as follows.

\par \textbf{Critic network.} In time slot $t$, the critic network takes the joint observations $o(t)=\{o_m(t)\mid\forall m\in \mathcal{M}\}$ and joint actions $a(t)=\{a_m(t)\mid \forall m\in \mathcal{M}\}$ of all agents as inputs. To enhance the stability of the training process, each critic network consists of two sub-networks, i.e., an online subnet ${Q}_{\psi_m}$ and a target subnet ${Q}_{\psi^\prime_m}$, which are parameterized by $\psi_m$ and $\psi^\prime_m$ respectively.
    
\par \textbf{Actor network.} The actor network receives the local observation $o_m(t)$ of an agent $m$ as input and outputs the corresponding action $a_m(t)$. Additionally, each actor network also consists of two subnets, i.e., an online subnet ${\pi}_{\phi_m}$ and a target subnet ${\pi}_{\phi^{\prime}_m}$, which are parameterized by $\phi_m$ and $\phi^\prime_m$, respectively.

\subsubsection{Network Training}

\par The training process of MADDPG involves iteratively updating the critic and actor networks to stabilize multi-agent learning and enhance overall performance.

\par \textbf{Actor Network Update.} The actor network is updated to maximize the expected return. To this end, the policy gradient is computed by using the deterministic policy gradient theorem. Specifically, we denote the policy objective function as $\mathcal{J}_{{\pi}_m}({\phi_m})$, and the gradient of the policy objective function with respect to $\phi_m$ is calculated as
\begin{equation}
    \label{eq_gradient_actor}
    \begin{aligned}
        \nabla_{\phi_m} \mathcal{J}_{{\pi}_{m}} ={}& \mathbb{E}_{\mathbf{s}, \mathbf{a} \sim \mathcal{D}} \left[ \nabla_{a_m} Q_m(\mathbf{s}, \mathbf{a} | \psi_m) \Big|_{a_i = \pi_i(o_i)} \right. \\
        & \left. \cdot \nabla_{\phi_m} \pi_m(o_m|\phi_m) \right],
    \end{aligned}
\end{equation}

\noindent where $\nabla_{a_m}$ represents the gradient vector of the policy objective function with respect to $\psi_m$, $Q_m(s(t),a_m(t)|\psi_m)$ denotes the state action value of agent $m$, which is calculated by the online subnet of the critic network $Q_{\psi_m}$, and $\pi_m(o_m(t)|\phi_m)$ is the output of actor network $\pi_m$ of agent $m$. Then, these gradients are back propagated to the online subnet of the actor network to update $\phi_m$, which is given as
\begin{equation}
    \label{eq_update_actor}
    \begin{aligned}
        \phi_m \leftarrow \phi_m+\zeta_\pi \nabla_{\phi_m} \mathcal{J}_{{\pi}_{m}},
    \end{aligned}
\end{equation}

\noindent where $\zeta_\pi\in (0,1]$ denotes the learning rate of the online actor network.

\par \textbf{Critic Network Update.} The critic network is updated by minimizing the loss function, which is typically the mean squared temporal difference (TD) error. Therefore, the mean squared error (MSE) can be given as
\begin{equation}
    \label{eq_error_critic}
    \begin{aligned}
        \text{MSE}(\psi_m)=\mathbb{E}_{\psi_m}[(Y_m(t)-Q_m(s(t),a_m(t)|\psi_m))^2],
    \end{aligned}
\end{equation}

\noindent where $Y_m(t)$ denotes the target value, given as
\begin{equation}
    \label{eq_target_critic}
    \begin{aligned}
        Y_m(t) = r_m(t) + \gamma Q^\prime_{m}(\mathbf{s}^\prime, \pi^\prime_{1}(o^\prime_{1}), \dots, \pi^\prime_{M}(o^\prime_{M}) | \psi^\prime_{m})
    \end{aligned}
\end{equation}

\noindent where ${\pi^\prime}_m(\cdot)$ and ${Q^\prime}_m(\cdot)$ are the outputs of the target subnet for the actor and critic network, respectively. Then, the gradient of $\text{MSE}(\psi_m)$ can be calculated as
\begin{sequation}
    \label{eq_gradient_critic}
    \begin{aligned}
        \nabla_{\psi_m}\text{MSE}(\psi_m) =& -2\mathbb{E}_{\psi_m}[r_m(t)+\gamma{Q^\prime}(S(t+1),a(t+1)|\psi^\prime_m)\\ &-{Q}(S(t),a(t)|{\psi}_m)]\nabla_{\psi_m}{Q}(S(t),a(t)|{\psi}_m).
    \end{aligned}
\end{sequation}

\noindent Accordingly, the critic network can be updated as
\begin{equation}
    \label{eq_update_critic}
    \begin{aligned}
        \psi_m \leftarrow \psi_m - \zeta_Q \nabla_{\psi_m} \text{MSE}(\psi_m),
    \end{aligned}
\end{equation}
\noindent where $\zeta_Q$ is the learning rate of the critic network.

\textbf{Target Subnet Update.} To stabilize the training process, the target subnets of critic network and actor networks are updated gradually to track the learned online networks. Specifically, the subnet of each critic network is upadated as follows:
\begin{equation}
    \label{eq_update_targetc}
    \begin{aligned}
        {\psi^\prime}_m\leftarrow\zeta_t{\psi^\prime}_m+(1-\zeta_t){\psi^\prime}_m,
    \end{aligned}
\end{equation}

\noindent where $\zeta_t$ means the update rate of the target subnet. Moreover, the subnet of each actor network is updated as follows:
\begin{equation}
    \label{eq_update_targeta}
    \begin{aligned}
        {\phi^\prime}_m\leftarrow\zeta_t{\phi^\prime}_m+(1-\zeta_t){\phi^\prime}_m.
    \end{aligned}
\end{equation}

\subsection{COCG Method}
\label{sec_COCG}

\par To address the challenges posed by the varying-dimensional and hybrid action space, we introduce the COCG method to enhance the MADDPG algorithm by deterministically deciding the computing resource allocation and UD association. Specifically, for the varying-dimensional computing resource allocation, we leverage convex optimization to derive a closed-form solution. Moreover, for the discrete UD association, we adopt a game-theoretic method, namely the coalitional game, to obtain a stable and mutually beneficial association structure in a distributed manner. The details are presented in the following subsections.

\subsubsection{Convex Optimization-based Computing Resource Allocation} 

\par Given the decisions of task offloading ratio $\bar{\mathbf{\Lambda}}$, UAV trajectory planning $\bar{\mathbf{Q}}$, and UD association $\bar{\mathbf{X}}$, while eliminating the irrelevant constant terms, UCMOP can be transformed into the computing resource allocation problem, which can be given as
\begin{subequations}
	\label{eq_problem1.1}
	\begin{alignat}{2}
		\mathbf{P1}: \quad & \min_{\mathbf{F}} \sum_{u \in \mathcal{U}} \sum_{i \in \mathcal{I}}\mathbb{I}_{\{\chi_i(t)=u\}}\eta_i(t)\lambda_i(t)\varpi_i(t)/f_{u,i}(t), \label{utility} \\
		\text{s.t.} \quad & \sum_{i\in \mathcal{I}}\mathbb{I}_{\{\chi_i(t)=u\}}f_{u,i}(t) \geq 0 , \ \forall u\in \mathcal{U}, i \in \mathcal{I},\label{pro1.1_c1}\\
            & \sum_{i\in \mathcal{I}}\mathbb{I}_{\{\chi_i(t)=u\}}f_{u,i}(t) \leq f_u^\text{max}. \ 
        \forall  u \in \mathcal{U}, \label{pro1.1_c2}
	\end{alignat}
\end{subequations}

\par Problem $\mathbf{P1}$ is a convex optimization problem, a property established in Theorem \ref{lemma_f_convex}. Consequently, we can apply convex optimization techniques~\cite{boyd2004convex} to solve it, and Theorem \ref{therem_opt_respource} provides the optimal computing resource allocation.

\begin{theorem}
\label{lemma_f_convex}
Problem $\mathbf{P1}$ is a convex optimization problem.
\end{theorem}

\begin{proof}
The detailed proof is provided in Appendix~A.
\end{proof}
\vspace{-0.5em}
\begin{theorem}
\label{therem_opt_respource}
The optimal computing resource allocation for problem $\mathbf{P1}$ is given as $\mathbf{F}^*=\{f_{u,i}^*(t)\mid\forall u\in \mathcal{U}, i\in \mathcal{I}, t\in\mathcal{T}\}$, where
\begin{equation}
    \label{eq_resource_allocation}
        \begin{aligned}
            f_{u,i}^*(t) = \frac{\sqrt{\frac{\eta_i(t) \lambda_i(t) \varpi_i(t)}{f_u^{\max}}}f_u^{\max}}{\sum_{i\in \mathcal{I}}\mathbb{I}_{\{\chi_i(t)=u\}} \sqrt{\frac{\eta_j(t) \lambda_j(t) \varpi_j(t)}{f_u^{\max}}}}.
        \end{aligned}
\end{equation}
\end{theorem}

\begin{proof}
The detailed proof is provided in Appendix~B.
\end{proof}
\vspace{-0.5em}

\subsubsection{Game Theory-based UD Association}

\par Given the optimal computing resource allocation, while removing irrelevant constant terms, UCMOP can be transformed into a UD association problem, which can be expressed as
\begin{subequations}
	\label{eq_problem1.2}
	\begin{alignat}{2}
		\mathbf{P2}: \quad & \min_{\mathbf{X}} \sum_{i\in \mathcal{I}} C_i^\text{off}(\mathbf{X}),
  \label{utility1} \\
		\text{s.t.} \quad & \eqref{pro_c2} \text{ and } \eqref{pro_c3} \label{p1.2_pro},
	\end{alignat}
\end{subequations}

\noindent where $C_i^\text{off}(\mathbf{X})$ represents the task offloading costs of UD $i$, which can be calculated as 
\begin{equation}
    \label{eq_off_C}
    \begin{aligned}
        C_i^\text{off}(\mathbf{X})=
        \begin{cases}
            C_i^\text{U},\chi_i=u,\forall u \in \mathcal{U},\\
            C_i^\text{N},\chi_i=n,\forall n \in \mathcal{N},
        \end{cases}
    \end{aligned}
\end{equation}

\noindent where $C_i^\text{U}$ and $C_i^\text{N}$ represent the task offloading costs when UD $i$ selects to offload the task to a UAV or a satellite, respectively, which can be calculated as
\begin{subequations}
	\label{eq_problem_C_US}
	\begin{alignat}{2}
		 C_i^\text{U}=\max\big({T_{i,u}^\text{Off},T_i^\text{LC}}\big)w^T + E_{i,u}^\text{Off}w^E,
  \label{eq_eq_off_C_UAV} \\
		C_i^\text{S}=\max\left({T_{i,n}^\text{Off},T_i^\text{LC}}\right)w^T + E_{i,n}^\text{Off}w^E .\label{eq_eq_off_C_satellite}
	\end{alignat}
\end{subequations}

\par In SAGIN-MEC system under consideration, each UAV or satellite may serve multiple UDs simultaneously due to the limited computing resources. However, this creates a competitive environment where UDs should offload their tasks while contending for the limited available computing resources, leading to inefficient resource utilization. To address this issue, the coalitional game provides a powerful and distributed framework for incentivizing UDs to collaborate and form coalitions, thereby sharing the resources of UAVs or satellites for mutual benefit \cite{sun2023bargain-match}. Specifically, the UDs can choose to associate with the same UAV or satellite by forming a coalition to minimize individual offloading costs. Moreover, the coalitional game supports dynamic coalition formation to enable UDs to join or leave coalitions as system conditions change. Additionally, coalitions can be restructured dynamically to adapt to network dynamics, UD mobility, UAV trajectories, or demand changes. Therefore, we present a coalitional game-based method for UD association, which is detailed as follows.

\par \textbf{Coalitional Game Formulation.} We model the process of the UD association in time slot $t$ as a coalitional game which is defined as a triplet $\Pi=\left(\mathcal{I},\mathcal{K}, \mathcal{X},\mathcal{U}\right)$, where 
\begin{itemize}
\item $\mathcal{I}=\{1,\ldots, I\}$ denotes a set of players, i.e., UDs.
\item $\mathcal{K}=\mathcal{U}\cup \mathcal{N}$ represents the set of MEC servers, i.e., UAVs and satellites.
\item $\mathcal{X}=\{\chi_i\mid \forall i \in\mathcal{I}, \chi_i\in\mathcal{K}\}$ is the set of UD associations.
\item $\mathcal{U}=\{U_i \mid \forall i \in \mathcal{I} \}$ represents the utility functions of UDs, which is used to evaluate the effectiveness of the UD association decision, i.e., $U_i=-C_i^\text{Off}$.
\end{itemize}

\par Each UD aims to minimize its cost by selecting an optimal UD association decision. Therefore, the coalitional game for UD association can be mathematically modeled as a distributed optimization problem as follows:
\begin{subequations}
	\label{eq_problem1.3}
	\begin{alignat}{2}
		\mathbf{P2}: \quad & \min_{\chi_i} \sum_{i\in \mathcal{I}} C_i^\text{off}(\chi_i,\chi_{-i}),
  \label{utility2} \\
		\text{s.t.} \quad & \eqref{pro_c2} \text{ and } \eqref{pro_c3},
        \label{p1.3_pro}
	\end{alignat}
\end{subequations}

\noindent where $\chi_{-i}$ is the set of decisions for all UDs, excluding $i$.  

\textbf{The Solution to the Coalitional Game.} To solve problem $\mathbf{P2}$, we first introduce the Nash equilibrium (NE), a stable situation where no UD is motivated to unilaterally alter its association decision, as given in Definition~\ref{def_NE}.
 
\begin{definition}
\label{def_NE}
    The profile of UD association decision $x^*=\left(\chi_1^*,\chi_2^*,\ldots,\chi_i^*\right)$ is an NE for $\Pi$ if and only if 
\begin{equation}
    \label{eq_nash}
    U_i(\chi_i^*,\chi_{-i}^*)\geq U_i(\chi_i,\chi_{-i}^*), \forall i\in\mathcal{I}, \forall \chi_i\in\mathcal{X}, \chi_i\neq \chi_i^*,
\end{equation}
\end{definition}

\par Then we elaborate on the coalition formation process. Specifically, each UD aims to minimize the cost by independently selecting either a UAV or a satellite for association. UDs that choose the same MEC server $k\in \mathcal{K}$ are considered to join the coalition of MEC server $k$, which can be represented as $ G_k = \{1,\ldots,i\}_{i\in \mathbb{I}_{\{\chi_i(t)=k\}}}$. Therefore, the UDs can be divided into $n$ disjoint coalitions due to the discrete feature of the UD association, which is given as follows:
\begin{equation}
    \label{eq_coalition_partition}
    \begin{aligned}
       \mathcal{G} =& G_1 \cup G_2 \cup \cdots \cup G_k, \\ & G_k \cap G_{k^\prime} = \emptyset, \, \forall k, k^\prime \in \mathcal{K}, \ k \neq k^\prime,
    \end{aligned}
\end{equation}

\noindent where $\Theta=\{G_1 \cup G_2 \cup \cdots \cup G_k,\}$ represents a coalition partition. For example, if solution $\Theta=\{\{1,3\}_1, \{2,4,5\}_2\}$ is achieved, the element $\{1,3\}_1$ indicates that UD 1 and UD 3 select UAV $1$ for association. Accordingly, the  coalition utility can be calculated as the sum of the utilities for all UDs in the coalition, which is given as
\begin{equation}
    \label{eq_coalition_utility}
    \begin{aligned}
        U_k^G=\sum_{i\in G_k} U_i(\chi_i,\chi_{-i}).
    \end{aligned}
\end{equation}

\par Moreover, to compare the performance relationship between two coalitions, we introduce the concept of preference coalition in Definition \ref{def_pre_coalition} \cite{Chen2021}.

\begin{definition}
\label{def_pre_coalition}
Given two coalitions $G_k$ and $G_k^\prime$, the coalition of MEC server $k$ (i.e., $G_k$) outperforms that of MEC server $k^\prime$ (i.e., $G_{k^\prime}$) if the former has higher coalition utility. Therefore, the preference coalition can be described as
\begin{equation}
    \label{eq_preference_coalition}
    \begin{aligned}
        G_k \succeq G_{k^\prime} \Leftrightarrow U^G_k \geq U^G_{k^\prime},
    \end{aligned}
\end{equation}

\noindent where $\succeq$ means that the performance of coalition $G_k$ is superior to that of coalition $G_{k^\prime}$. In other words, UDs prefer to join coalition $G_k$ rather than $G_{k^\prime}$.
 \end{definition}

\par Guided by the preference relation, each UD can decide to join or leave a coalition by adhering the switch rule and insert rule as follows.

\textit{Switch rule:} UDs $i$ ($i \in G_k$) and $i^\prime$ ($i^\prime \in G_{k^\prime}$) prefer to switch their coalitions, i.e., $i \in G_{k^\prime}$ and $i^\prime \in G_k$, when the following conditions are satisfied:
\begin{sequation}
    \label{eq_switch_rule}
    \begin{aligned}
        \sum_{j \in \{G_k \setminus \{i \cup i^\prime\}\}} U_j(\chi_j, \chi_{-j}) + \sum_{j\prime \in \{G_{k^\prime} \setminus \{i^\prime \cup i\}\}} U_j(\chi_j, \chi_{-j}) \\\geq U^G_k + U^G_{k^\prime}.
    \end{aligned}
\end{sequation}

\textit{Insert rule:} UD $i$ prefers to leave coalition $G_k$ and join coalition $G_{k^\prime}$ if the following condition is satisfied:
\begin{equation}
    \label{eq_insert_rule}
    \begin{aligned}
        \sum_{j \in \{G_k \setminus \{i\}\}} U_j(\chi_j, \chi_{-j}) + \sum_{j \in \{G_{k^\prime} \cup \{i\}\}} U_j(\chi_j, \chi_{-j}) \\\geq U^G_k + U^G_{k^\prime}.
    \end{aligned}
\end{equation}

\par The aforementioned rules ensure that the total utility of the two new coalitions is greater than that of the original ones after performing the switch or insert operations. As a result, the UD association decision can be obtained by iteratively update the association decisions of the UDs through applying switch or insert rules until the NE is reached. Moreover, the UD association result is an NE, which is given as in Theorem \ref{theorem_ns}.

\begin{theorem}
\label{theorem_ns}
The UD association decision obtained by the coalitional game is an NE.
\vspace{-0.5em}
\end{theorem}

\begin{proof}
The detailed proof is provided in Appendix~C.
\end{proof}
\vspace{-0.5em}

\par The process of game theory-based UD association is described in Algorithm~\ref{al_cfg}. First, the coalition partition is initialized, where the UDs located within the coverage area of a UAV are assigned to the coalition of that UAV and the UDs outside the coverage of any UAVs are assigned to the coalition of the nearest satellite (line~\ref{cfg:line:init}). Then, in each iteration, a UD $i$ is selected from its current coalition $G_k$, and it explores a new coalition $G_{k^{\prime}}$ for potential association (line~\ref{cfg:line:selecti1}). Then, the coalition structure is iteratively refined through a distributed process where each UD autonomously seeks to improve its own utility. Subsequently, another UD $i^{\prime}$ is selected from $G_{k^{\prime}}$ to facilitate a utility comparison with UD $i$ (line~\ref{cfg:line:selecti2}). Next, if the exchange rule is satisfied, the coalitions $G_k$ and $G_k^\prime$ are swapped (lines~\ref{cfg:line:switch_start} to \ref{cfg:line:switch_end}). Finally, UD $i$ decides whether to join a new coalition based on the insertion rule (lines~\ref{cfg:line:insert_start} to \ref{cfg:line:insert_end}). Repeat the above steps until no UD deviates from its current coalition.

\begin{algorithm}[]
    \caption{Game Theory-based UD Association}
    \label{al_cfg}
    \KwIn{The offloading ratio $\lambda_i$ of UDs, and the locations  of UAVs $\mathbf{q}_{u}(t)$ and UDs $\mathbf{q}_{i}(t)$. }
    \KwOut{UD association $\mathbf{X}^{*}$}
        Initialize coalition partition $\mathcal{G}$;\label{cfg:line:init} 
    \Repeat{The partition converges to a stable state}{
        Randomly select UD $i$ and denote the coalition that it belongs to as $G_k$;\\\label{cfg:line:selecti1}
        Randomly select another coalition $G_{k^\prime}$ from the nominee set $\mathcal{K}_i$, $G_k \neq G_{k^\prime}$ of $i$;\\ \label{cfg:line:selecti2} 
        Calculate the coalition utility $U^G_{k}$ and $U^G_{k^\prime}$;\\
        Randomly select a UD $i^\prime$ from coalition $G_{k^\prime}$;\\
        \eIf{Switch rule (28) is satisfied}{
            $G_{k^\prime} \leftarrow \{G_{k^\prime} \setminus i^\prime\} \cup i$; \\\label{cfg:line:switch_start}
            $G_k \leftarrow \{G_k \setminus i\} \cup i^\prime$; \\ \label{cfg:line:switch_end}
        }{
            Satisfies insert rule (29)\\
            $G_{k^\prime} \leftarrow G_{k^\prime} \cup i$; \\\label{cfg:line:insert_start}
            $G_k \leftarrow G_k \setminus i$;
        }\label{cfg:line:insert_end}
    }
\end{algorithm}

\subsection{Main Steps of MADDPG-COCG Algorithm and Performance Analysis}

\par In this section, the main steps and complexity of the MADDP-COCG algorithm is presented.

\subsubsection{Main Steps of MADDPG-COCG Algorithm}

\par The overall procedure of the proposed MADDPG-COCG algorithm is summarized in Algorithm \ref{alg:maddpg_cocg}. The proposed MADDPG-COCG algorithm determines the final decisions for task offloading, UAV trajectory, resource allocation, and UD association by integrating the COCG method with MADDPG. During the training phase, agents (UAVs and UDs) interact with the environment. Each agent selects an action based on its local observation and policy. The environment then provides a joint reward and the next observation after executing these actions along with the deterministic decisions from the COCG method. Then the experience is stored in the replay buffer. For network updates, a batch of experiences is randomly sampled from the buffer. The critic network is updated by minimizing the temporal difference error, followed by a periodic update of the actor network. Finally, the target networks are softly updated to ensure training stability.

\begin{algorithm}[]
		\raggedright
		\caption{MADDPG-COCG Algorithm}
		\label{alg:maddpg_cocg}
		\LinesNumbered
            Initialize an experience replay buffer $\mathcal{D}$; \\
            \For{\rm{m} = $1$ \KwTo $M$}{
                Initialize the online networks $Q_{\psi_m}$ and $\pi_{\phi_m}$ with parameters $\psi_m$ and $\phi_m$; \\
                Initialize the target networks $Q_{\psi^\prime_m}$ and $\pi_{\phi^\prime_m}$ with parameters $\psi^\prime_m$ and $\phi^\prime_m$; \\
            }
		
		\For{\rm{each episode}}{
			Reset the initial environment state $\boldsymbol{s}(0)$;\\

			\Repeat{environment is terminated or $step \geq T$}{
				$step \gets 0$;\\
                    \For{\rm{m} = $1$ \KwTo $M$}{
                       Receive the observation $o_m(step)$; \\
                       Select action $a_m(step)$ according to the policy $\pi_m(o_m(step))|\phi_m)$ of agent $m$; \\
                    }
                    Solve computing resource allocation and UDs association problem by calling \textbf{Algorithm~\ref{al_cfg}}; \\
                    Execute the joint action $a(step)$;\\
                    The environment returns the joint reward $r(step)$ and the next state $s(step+1)$, according to the results of \textbf{Algorithm~\ref{al_cfg}}; \\
                    Add the experience tuple $(s(step), a(step), r(step), s(step+1))$ to the relay buffer $\mathcal{D}$; \\
                    Randomly choose a batch $\mathcal{B}$; \\
                    Update the parameters of the online subnet for the critic network according to Eq.~\eqref{eq_update_critic}; \\
                    \If{$step\mod d$}
				{
					Update the parameters of the actor network based on Eq.~\eqref{eq_update_actor};\\
					Soft-update the online target subnets based on Eqs.~\eqref{eq_update_targetc} and \eqref{eq_update_targeta};}
                    
			}		
		}
\end{algorithm}

\subsubsection{Complexity Analysis}
\par The complexity of the proposed MADDPG-COCG algorithm is analyzed in two stages: training and execution~\cite{zhang2025multi-objective-diffusion}.

\textbf{Training Phase:} During the training phase, the computational complexity is primarily determined by experience collection and network updates at each time step. Experience collection requires $\mathcal{O}(M|\phi_m| + \text{Iter}_{cg} I (U+N))$ for action selection (forward pass) and COCG execution for $M$ agents. Network updates have a complexity of $\mathcal{O}(\frac{1}{d}MB|\phi_m| + MB|\psi_m|)$ for the actor and critic networks. The overall space complexity is dominated by the network parameters and the replay buffer, totaling $\mathcal{O}(M(|\phi_m|+|\psi_m|) + D(|s|+|a|+M))$.

\textbf{Execution Phase:} The computational complexity during execution per time step is $\mathcal{O}(M|\phi_m| + \text{Iter}_{cg} I (U+N))$, which comprises the forward pass of $M$ actor networks and the execution of the COCG method. The corresponding space complexity is $\mathcal{O}(M|\phi_m|)$ to store the actor network parameters for all agents.

\section{Simulation Results}
\label{sec_simulation}

\subsection{Simulation Setup}
\label{simulation_set_up}

\par In this section, we introduce the simulation setup, including the system scenario, key parameters, benchmarks, and performance metrics.

\subsubsection{Scenarios} 

\par We consider an SAGIN-MEC system where 3 UAVs and 1 LEO satellite are deployed to provide computing services for 15 UDs in a $1000\times 1000 \ \text{m}^2$ rectangular area. Moreover, the time horizon is set to $100$ s, and it is divided into $T = 100$ time slots, each with a equal duration of $\tau = 1$ s.

\subsubsection{Parameters}

\par For the satellite, the orbital altitude is set as $1000$ km~\cite{Chai2023}. Moreover, for the UAVs, the initial positions are set as $q_1^I=[150-250, 150-250]$, $q_2^I=[750-850,150-250]$ and $q_3^I=[450-550,750-850]$, respectively, and the fixed altitude is set as $H$ = 100 m~\cite{Sun2024TJCCT}.The other parameters are set to the default values in Table~\ref{tab_simuParameter}. 
\begin{table}[t]
	\setlength{\abovecaptionskip}{3pt}%
	\setlength{\belowcaptionskip}{0pt}%
	\caption{Simulation parameters}
	\label{tab_simuParameter}
	\renewcommand*{\arraystretch}{1.2}
	\begin{center}
		\begin{tabular}{p{.07\textwidth}|p{.23\textwidth}|p{.12\textwidth}}
			\hline
			\hline
			\textbf{Symbol}&\textbf{Description}&\textbf{Default value}\\
			
			\hline
				    $v_U^\text{max}$&Maximum velocity constraint for the UAVs &25 m/s \cite{Zeng2019}\\
			\hline
				    $\varpi_i(t)$&Task size& [1, 5] Mb \\
			\hline
				    $\eta_{i}(t)$&Computation density & [1000, 1500] cycles/Byte \\
			\hline
   				    $T_{i}^{\text{max}}(t)$&Maximum tolerable delay & [1, 5] s \\
			\hline
                    $f_{i}^{\text{loc}}$&Computing resources of UD $i$ & 0.3 GHz \\
			\hline
                     $f_{u}^{\text{max}}$&Computing resources of UAV $u$&[2, 4] GHz\\

			\hline
                     $p_{i,u}$, $p_{i,n}$&Transmit power of UD $i$& [20, 25] dBm\\	
			\hline
			         $B_u^\text{total}$&Total bandwidth of UAV $u$& 10 MHz\\
			\hline
				   $B_{i,n}(t)$ &Bandwidth between UD $i$ and satellite $n$ &1 MHz\\
			  \hline 	
				  $\sigma^2$& Noise power&$1.58 \times 10^{-13}$ W \\
			  \hline
				 $\epsilon_1$, $\epsilon_2$& Environment-dependent variables for LoS probability & 4.88, 0.43\\
			  \hline
			      $\rho^L$, $\rho^N$ &Additive path loss for LoS and NLoS links & 0.1 dB, 21 dB\\
			  \hline
			  	$\delta_1$, $\delta_2$, $\delta_3$, $\delta_4$&Nakagami fading parameters & 4, 2, 3, 1 \cite{Zhang2018Dense}\\
			  \hline
				  $E_u^{\text{max}}$& Energy constraint of UAV $u$&36 kJ\\
			  \hline
				  $E_i^{\text{max}}$& Energy constraint of UD $i$&1 (Wh/GHz)\\
			  \hline
				  $\gamma$&Discount factor&0.95\\
			   \hline
			         $\bar{d}$& Frequency of policy updates& 5\\
			   \hline
			  	 $\zeta_{\pi}$, $\zeta_{Q}$&Learning rates of the actor network and critic network&$5\times10^{-4}$, $5\times10^{-4}$\\
                 \hline
			  	 $\zeta_{t}$ &Soft update rate of the target networks&$5\times10^{-3}$\\
			   \hline
		\end{tabular}
	\end{center}
\end{table}

\subsubsection{Benchmarks} 
\par This work evaluates the proposed MADDPG-COCG algorithm in comparison with the following baselines. 
\begin{itemize}
        \item \textit{Equal computing resource allocation (ECRA)}: The available computing resources of each UAV are equally allocated to the requested UDs, while the decisions of UD association, task offloading ratio, and UAV trajectory panning are determined based on the proposed MADDPG-COCG algorithm.
        
        \item \textit{Nearby offloading (NO)}: Each UD is associated with the nearby server, while the decisions of computing resource allocation, task offloading ratio, and UAV trajectory panning are determined based on the proposed MADDPG-COCG algorithm.
        
        \item \textit{Conventional MADDPG \cite{Du2024MADDPG}}: The UAV agents learn trajectory planning decisions, while the UD agents determine the decisions of UD association, task offloading ratio, and computing resource allocation. Note that, to address the issue of varying dimensions in the action and state spaces caused by changes in the number of associated UDs when UAVs make resource allocation decisions, the UD agents handle resource request decisions, and the UAVs allocate the computing resources based on the requests submitted by the UD agents.
        
        \item \textit{Multi-agent proximal policy optimization (MAPPO) \cite{Li2025Collaborative}}: The MAPPO is employed to optimize the computing resource allocation, UD association, task offloading ratio, and UAV trajectory planning.
\end{itemize}

\subsubsection{Performance Indicators}

To evaluate the overall performance of the proposed MADDPG-COCG algorithm, we adopt the following metrics:

\begin{itemize}
    \item \textit{Aggregated UD Cost}: $ \sum_{t\in\mathcal{T}} C(t)$, which indicates the aggregated cost of all UDs.

    \item \textit{Average Task Completion Delay}: $\frac{1}{T}\sum_{t\in\mathcal{T}}\frac{1}{I}\sum_{i\in\mathcal{I}}T_i(t)$, which indicates the average delay for completing each task.

    \item \textit{Average UD Energy Consumption}: $\frac{1}{T}\sum_{t\in\mathcal{T}}\frac{1}{I}\sum_{i\in\mathbb{I}}E_{i}(t)$, which indicates the average energy consumption of UDs per unit time.

    \item \textit{Average UAV Energy Consumption}: $\frac{1}{T}\sum_{t\in\mathcal{T}}\frac{1}{I}\sum_{u\in\mathcal{U}}(E_u(t)  +\sum_{i\in\mathbb{I}_{\{\chi_i(t)=u\}}}E_{u}(t))$, which indicates the average energy consumption of UAVs per unit time.
\end{itemize}

\subsection{Evaluation Results}
\label{Numerical Results}

\par In this section, we first examine the system performance of the proposed MADDPG-COCG algorithm with default parameter settings. Then, we analyze the performance impact of varying parameters on both the proposed MADDPG-COCG algorithm and the benchmarks. Finally, we evaluate the effectiveness of the MADDPG-COCG.

\begin{figure*}[!hbt] 
	\centering
		\setlength{\abovecaptionskip}{2pt}%
		\setlength{\belowcaptionskip}{2pt}%
	\subfigure[Aggregated UD cost]
	{
		\begin{minipage}[t]{0.23\linewidth}
			\raggedleft
			\includegraphics[width=1.75in]{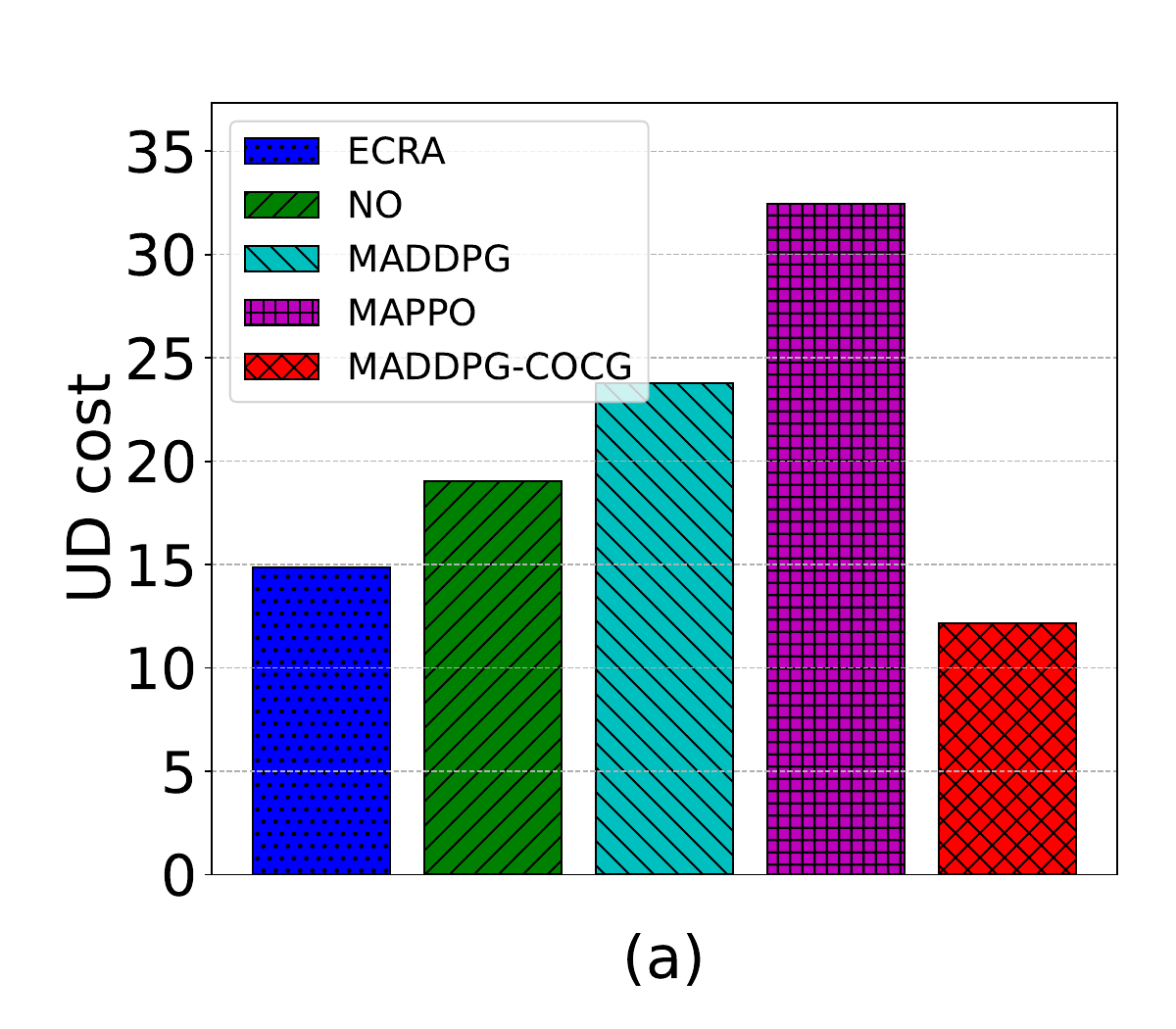}
		\end{minipage}
	}
	\subfigure[Average task completion delay]
	{
		\begin{minipage}[t]{0.23\linewidth}
			\centering
			\includegraphics[width=1.75in]{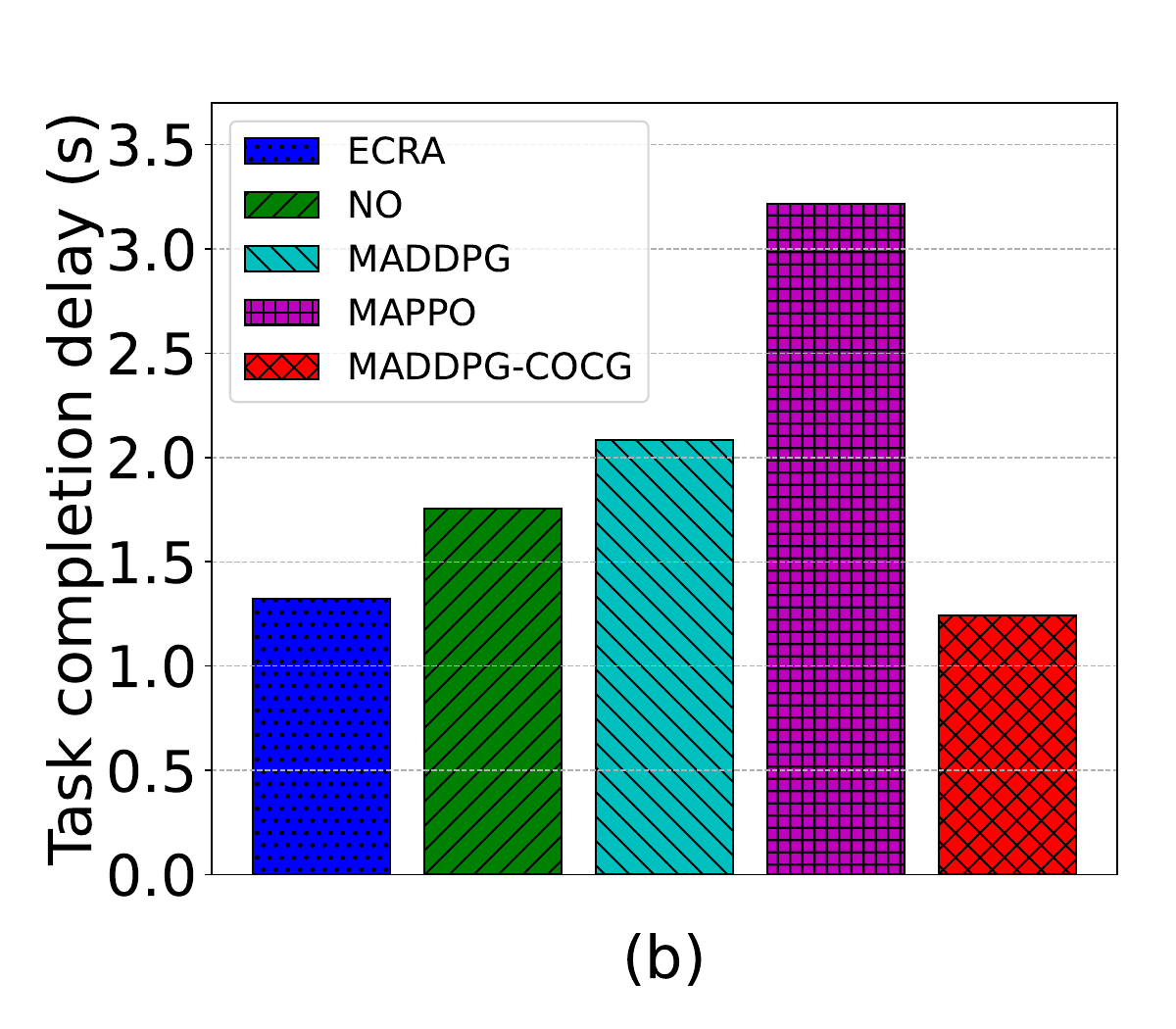}	
		\end{minipage}
	}
	\subfigure[Average UD energy consumption]
	{
		\begin{minipage}[t]{0.23\linewidth}
			\centering
			\includegraphics[width=1.75in]{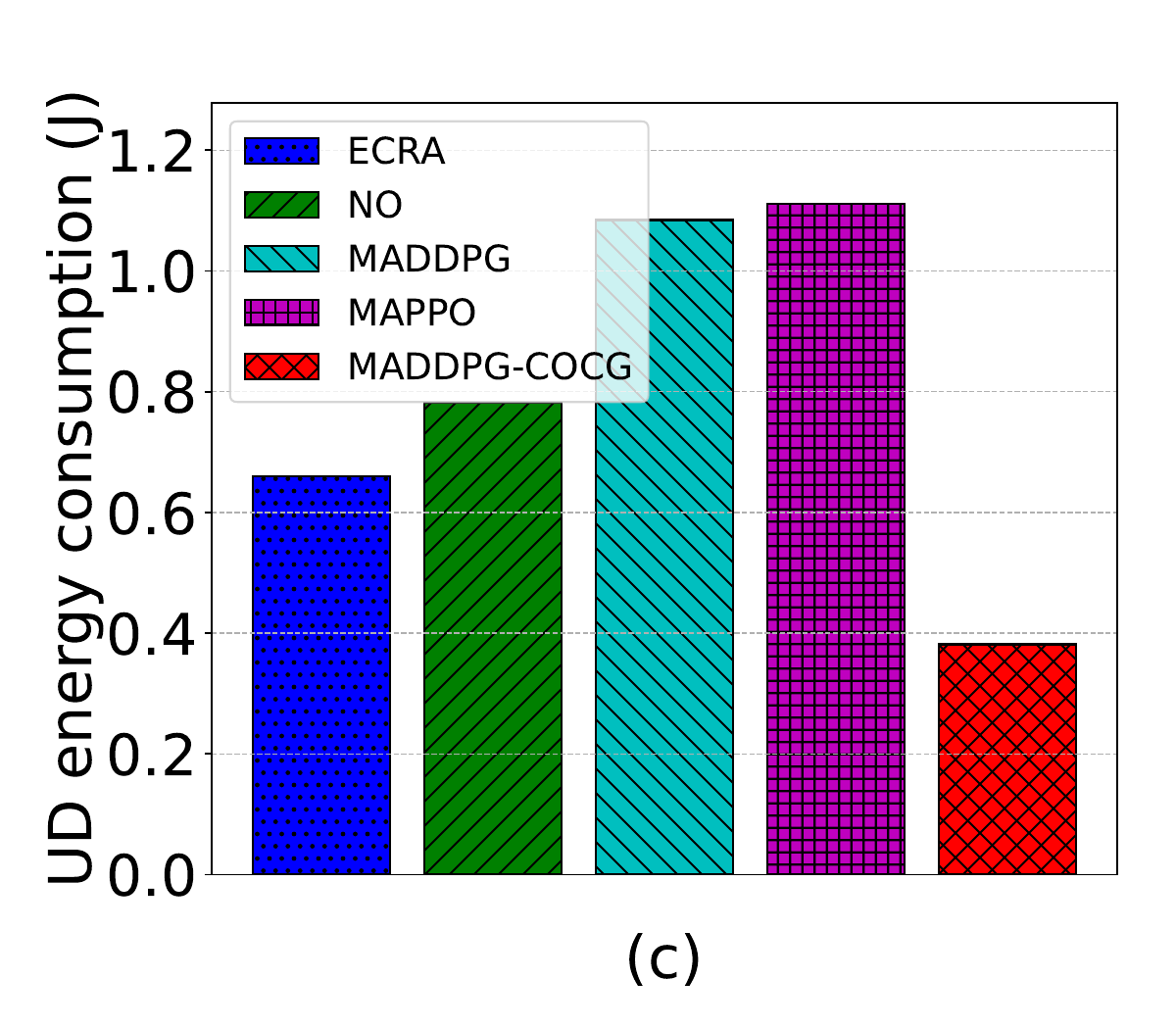}
		\end{minipage}
	}
        \subfigure[Average UAV energy consumption]
	{
		\begin{minipage}[t]{0.23\linewidth}
			\centering
			\includegraphics[width=1.75in]{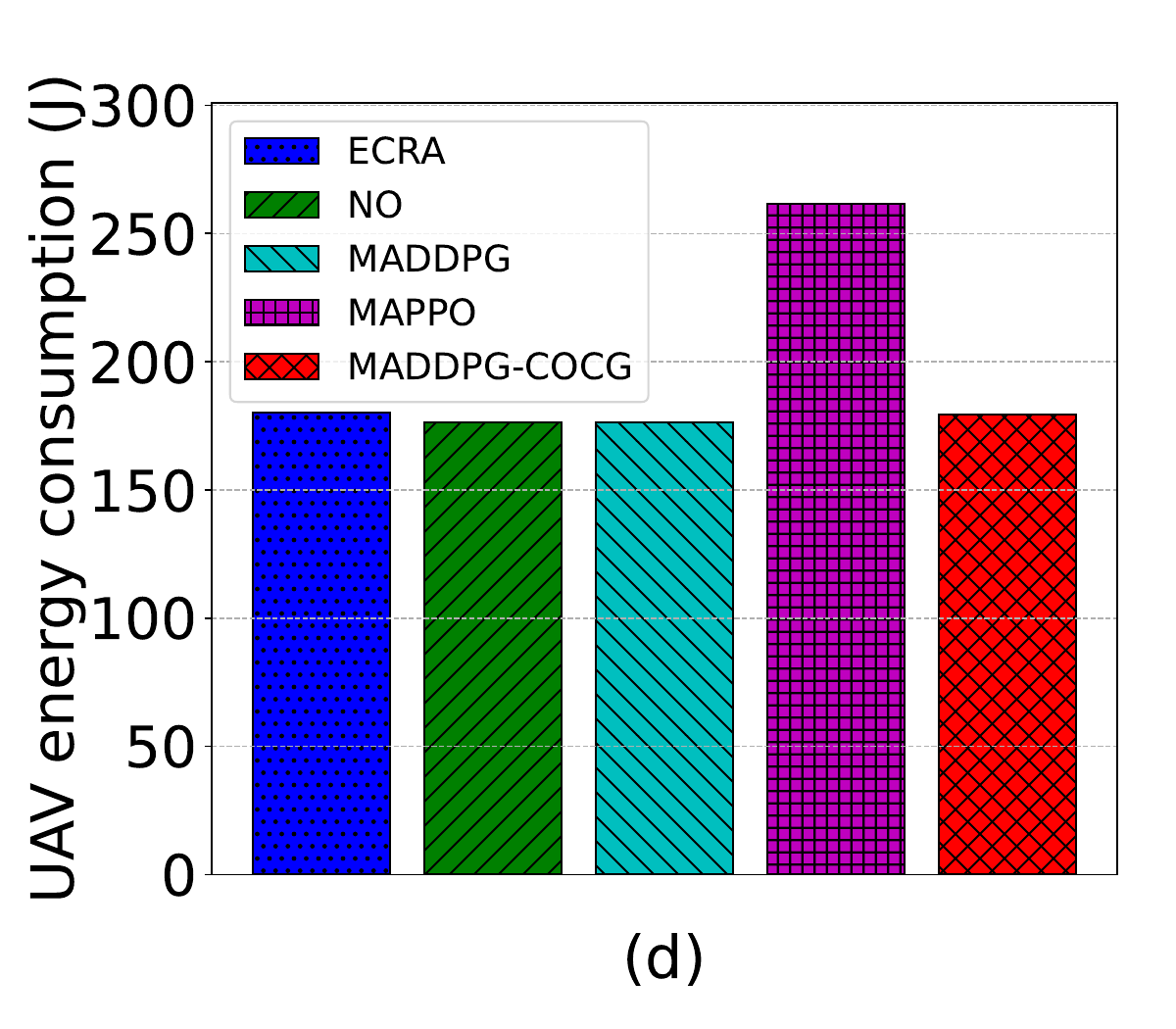}
		\end{minipage}
	}
	\centering
	\caption{System performance comparison among different algorithms.}
	\label{fig_Performance}
     \vspace{-1.5em}
\end{figure*}

\subsubsection{Algorithm Performance Evaluation}

\par Fig. \ref{fig_Performance} presents the performance comparison among the different algorithms under the default parameter settings. Specifically, it can be observed from Figs. \ref{fig_Performance}(a), \ref{fig_Performance}(b), and \ref{fig_Performance}(c) that the proposed MADDPG-COCG algorithm significantly outperforms the benchmarks in terms of aggregated UD cost, average task completion delay, and average UD energy consumption. Moreover, Fig.~\ref{fig_Performance}(d) exhibits the MADDPG-COCG algorithm incurs a slight increase in UAV energy consumption compared to ECRA, NO, and MADDPG, while remaining considerably more energy-efficient than MAPPO. This is because in the considered SAGIN-MEC system, most tasks are delay-sensitive, and mobile UDs are constrained by limited battery capacities due to their physical size. Therefore, the formulated UCMOP adopts a user-centric design that aims to minimize UD cost by jointly considering UD energy consumption and task completion delay, while ensuring UAVs operate within feasible energy constraints. In other words, the objective of MADDPG-COCG is to improve user-centric performance under UAV energy constraints. This reveals a practical trade-off, where a modest expenditure of UAV energy achieves substantial improvements in service quality for users. Accordingly, the proposed MADDPG-COCG algorithm achieves superior UD performance in terms of cost, delay, and energy consumption, while maintaining relatively efficient UAV energy usage comparable to other benchmarks.

\vspace{-1.0em}
\begin{figure}[!hbt] 
    \centering
    \setlength{\abovecaptionskip}{2pt}%
    \setlength{\belowcaptionskip}{2pt}%
    \includegraphics[width =3.5in]{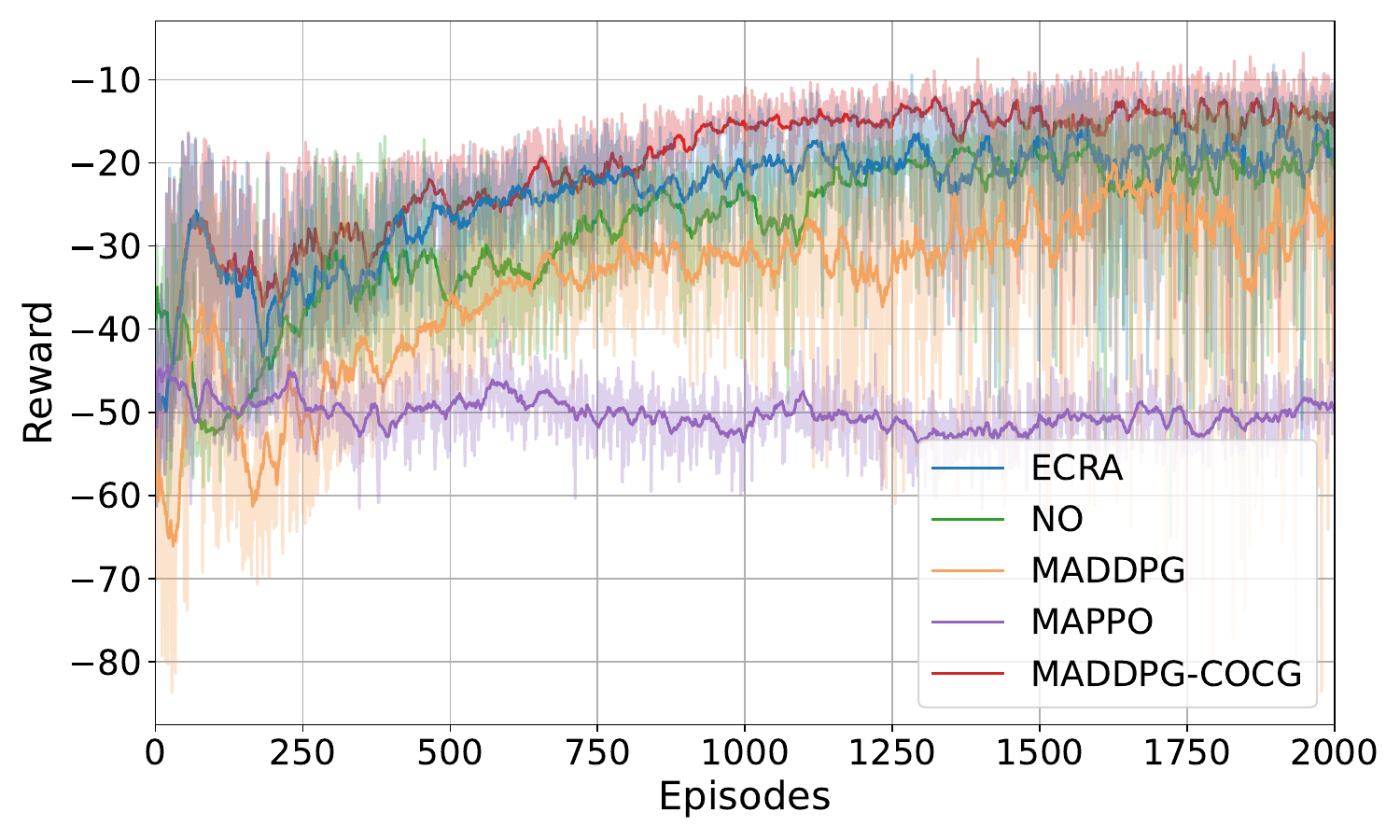}
    \caption{Training performance.}
    \label{fig_reward}
    \vspace{-0.5em}
\end{figure}
\subsubsection{Convergence and Stability Performance}

\par Fig. \ref{fig_reward} illustrates the convergence performance of different algorithms during the training. As the training progresses, the five algorithms exhibit gradual performance improvement, with the rewards stabilizing after approximately 1500 episodes. Specifically, ECRA and NO show suboptimal convergence, achieving 30\% lower rewards than the proposed MADDPG-COCG algorithm after stabilization. This is due to the fixed decisions of computing resource allocation and UD association without full optimization. Moreover, MADDPG shows relatively high variance and unstable learning performance throughout the training phase. This is mainly because the original MADDPG, which handles the hybrid and dynamic action space directly, suffers from unstable exploration and inefficient policy learning. Additionally, MAPPO fails to converge effectively, which can be attributed to its on-policy learning mechanism and insufficient sample efficiency in dynamic multi-agent environments. In summary, these results validate that the proposed MADDPG-COCG algorithm effectively addresses the challenges of hybrid and variable action space decision-making, achieving superior stability, robustness, and convergence compared to baselines.

\begin{figure*}[!hbt] 
	\centering
		\setlength{\abovecaptionskip}{2pt}%
		\setlength{\belowcaptionskip}{2pt}%
	\subfigure[Aggregated UD cost]
	{
		\begin{minipage}[t]{0.23\linewidth}
			\raggedleft
			\includegraphics[width=1.75in]{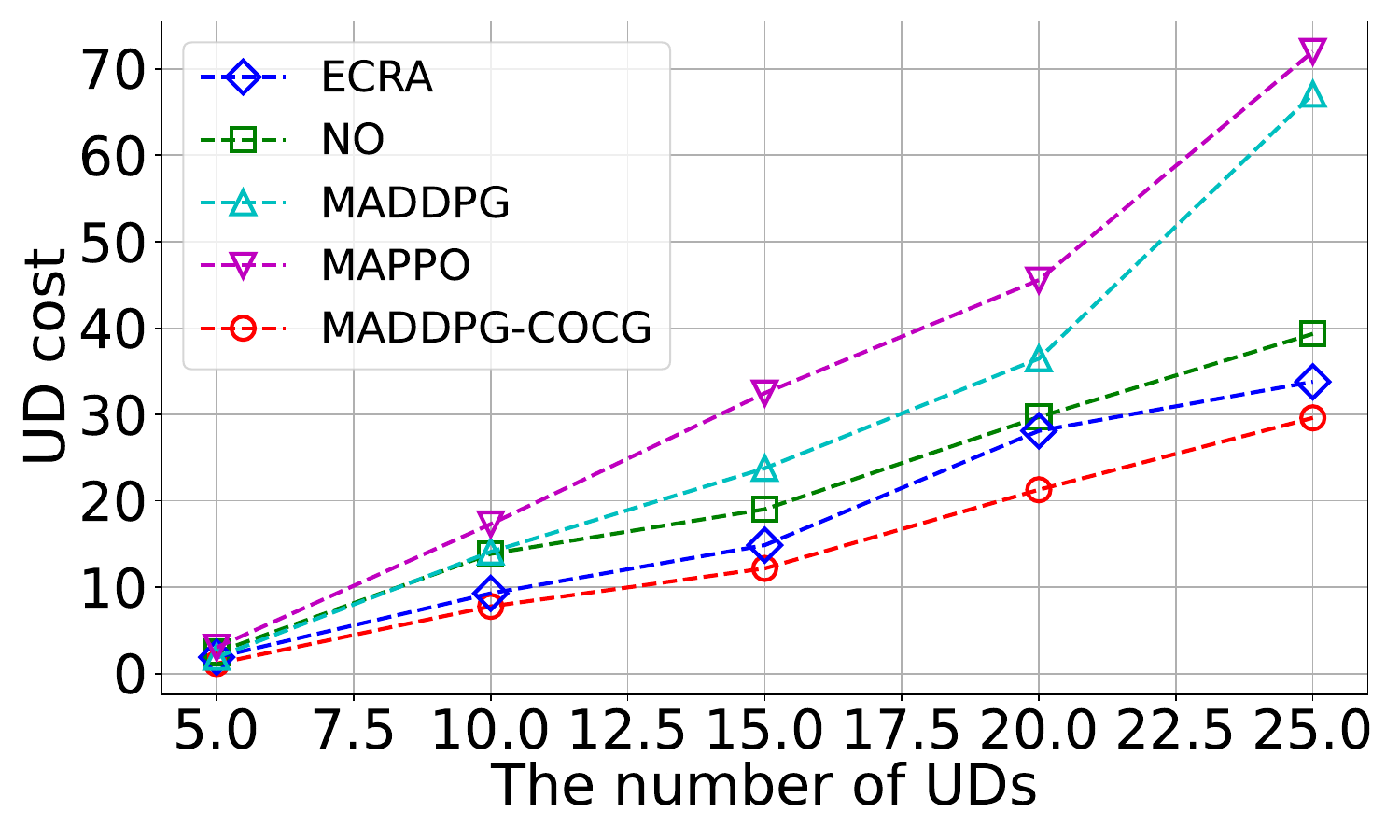}
		\end{minipage}
	}
	\subfigure[Average task completion delay]
	{
		\begin{minipage}[t]{0.23\linewidth}
			\centering
			\includegraphics[width=1.75in]{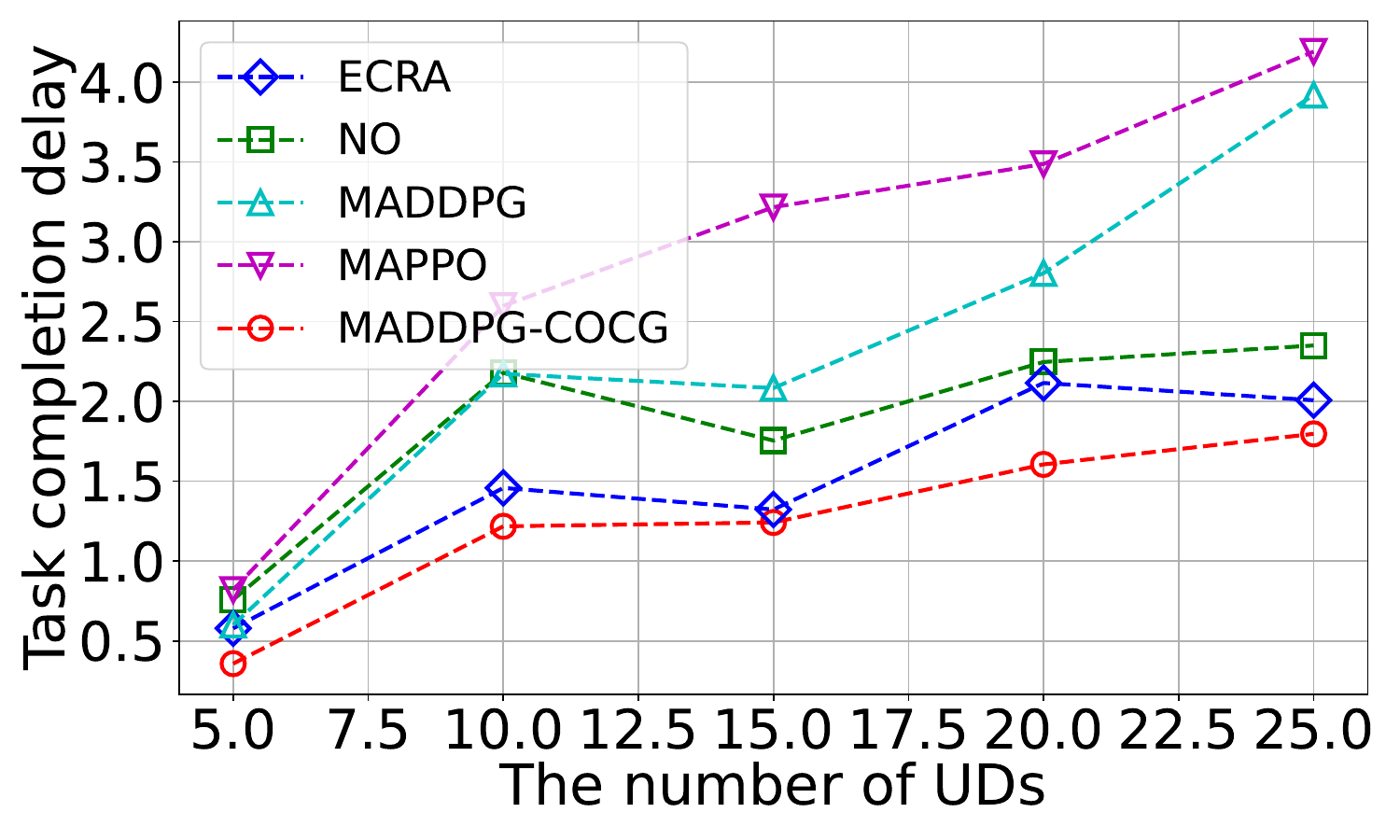}	
		\end{minipage}
	}
	\subfigure[Average UD energy consumption]
	{
		\begin{minipage}[t]{0.23\linewidth}
			\centering
			\includegraphics[width=1.75in]{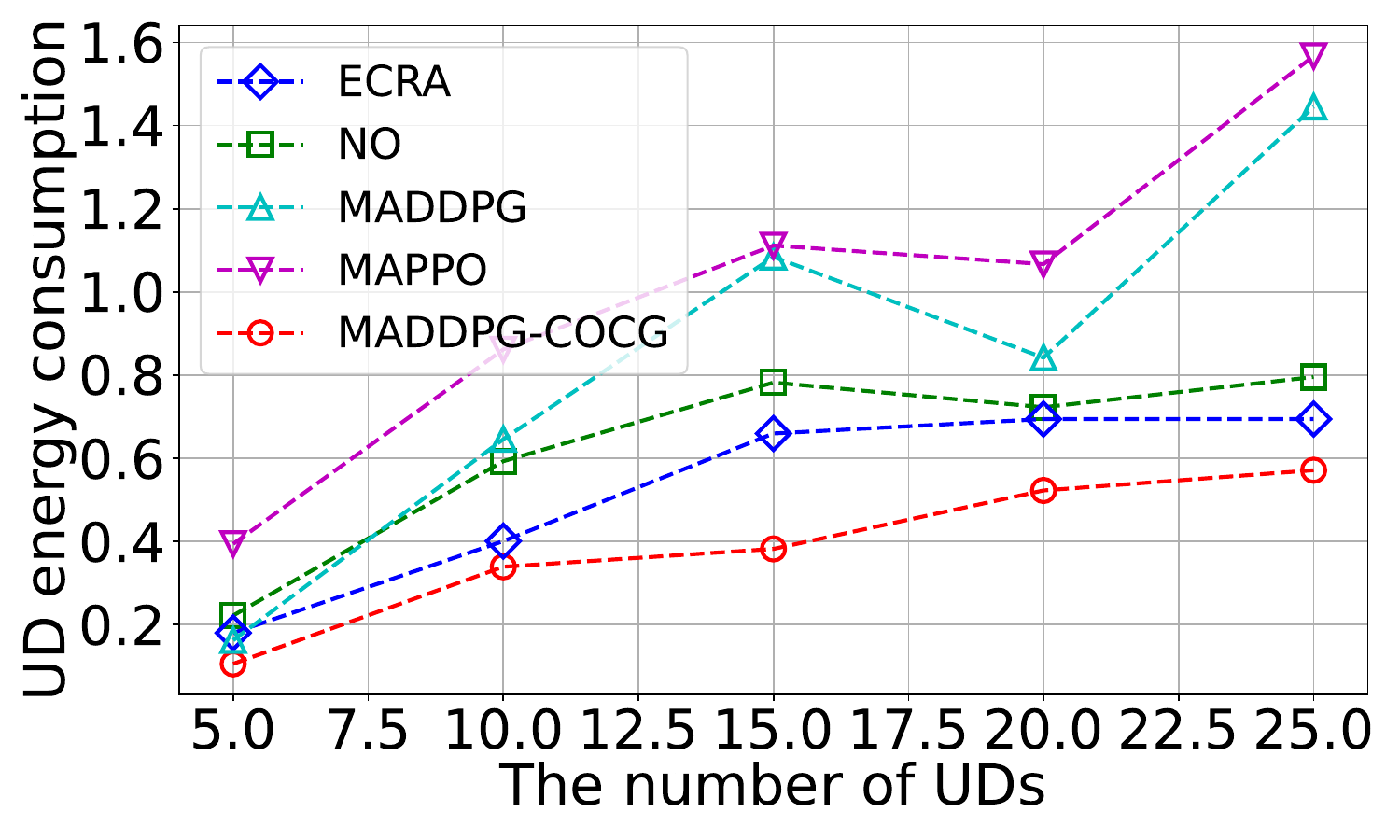}
		\end{minipage}
	}
        \subfigure[Average UAV energy consumption]
	{
		\begin{minipage}[t]{0.23\linewidth}
			\centering
			\includegraphics[width=1.75in]{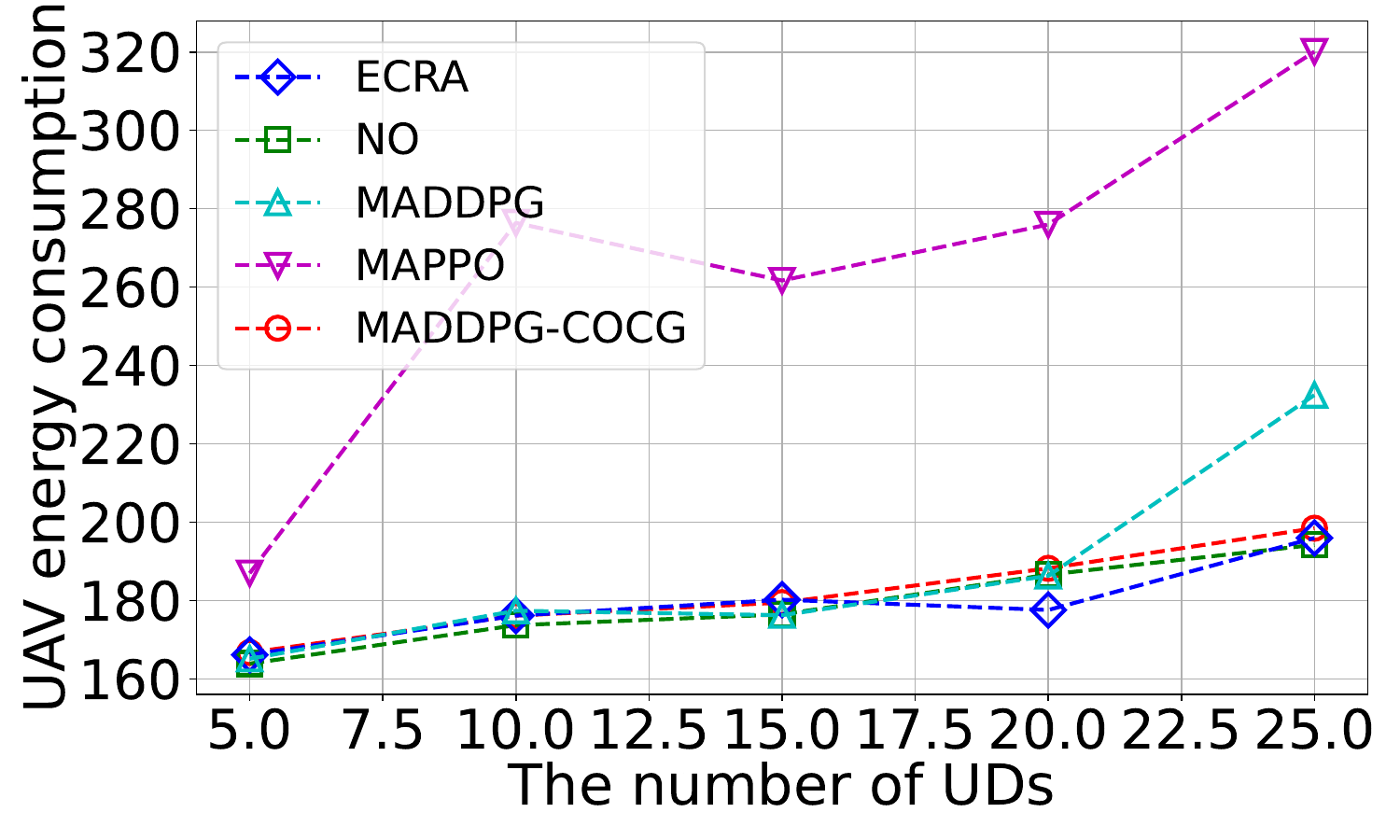}
		\end{minipage}
	}
	\centering
	\caption{System performance with the number of UDs.}
	\label{fig_UD}
	\vspace{-1.5em}
\end{figure*}

\subsubsection{Impact of System Settings}

\par We evaluate the performance of the MADDPG-COCG algorithm and benchmarks under different system settings, including the number of UDs, task sizes, and computing resources of UAVs.

\par \textbf{Impact of UD numbers.} 
Figs. \ref{fig_UD}(a)-(d) illustrate the effect of varying UD numbers on aggregated UD cost, average task completion delay, and average energy consumption for both UDs and the UAV. As the number of UDs increases, all metrics for all algorithms rise, a natural consequence of growing computational and offloading demands. Notably, the performance of MADDPG and MAPPO degrades and shows significant fluctuations. This is primarily due to the challenges traditional DRL algorithms face in large, mixed discrete-continuous action spaces. The growing number of agents introduces environmental dynamics and exponentially increases the state-action space, hindering efficient exploration and slowing convergence.

\par In contrast, our proposed MADDPG-COCG algorithm consistently outperforms the benchmarks in terms of aggregated UD cost, task delay, and UD energy consumption, while remaining competitive in UAV energy consumption. This superiority stems from its efficient hybrid strategy, which integrates convex optimization for resource allocation, a coalitional game for UD association, and MADDPG for joint task offloading and trajectory planning. Notably, while MADDPG-COCG consumes slightly more UAV energy in high-density scenarios than ECRA and NO, this trade-off is justified. The extra energy is used for more dynamic trajectory adjustments and resource engagement, which significantly reduces task delay and UD energy consumption, thereby improving overall service quality. Furthermore, MADDPG-COCG shows greater stability, evidenced by smoother performance curves with less variance. This indicates superior policy robustness and generalization across different system scales. In summary, these results confirm the effectiveness and scalability of our MADDPG-COCG algorithm, especially in dense UD environments.

\begin{figure*}[!hbt] 
	\centering
		\setlength{\abovecaptionskip}{2pt}%
		\setlength{\belowcaptionskip}{2pt}%
	\subfigure[Aggregated UD cost]
	{
		\begin{minipage}[t]{0.23\linewidth}
			\raggedleft
			\includegraphics[width=1.75in]{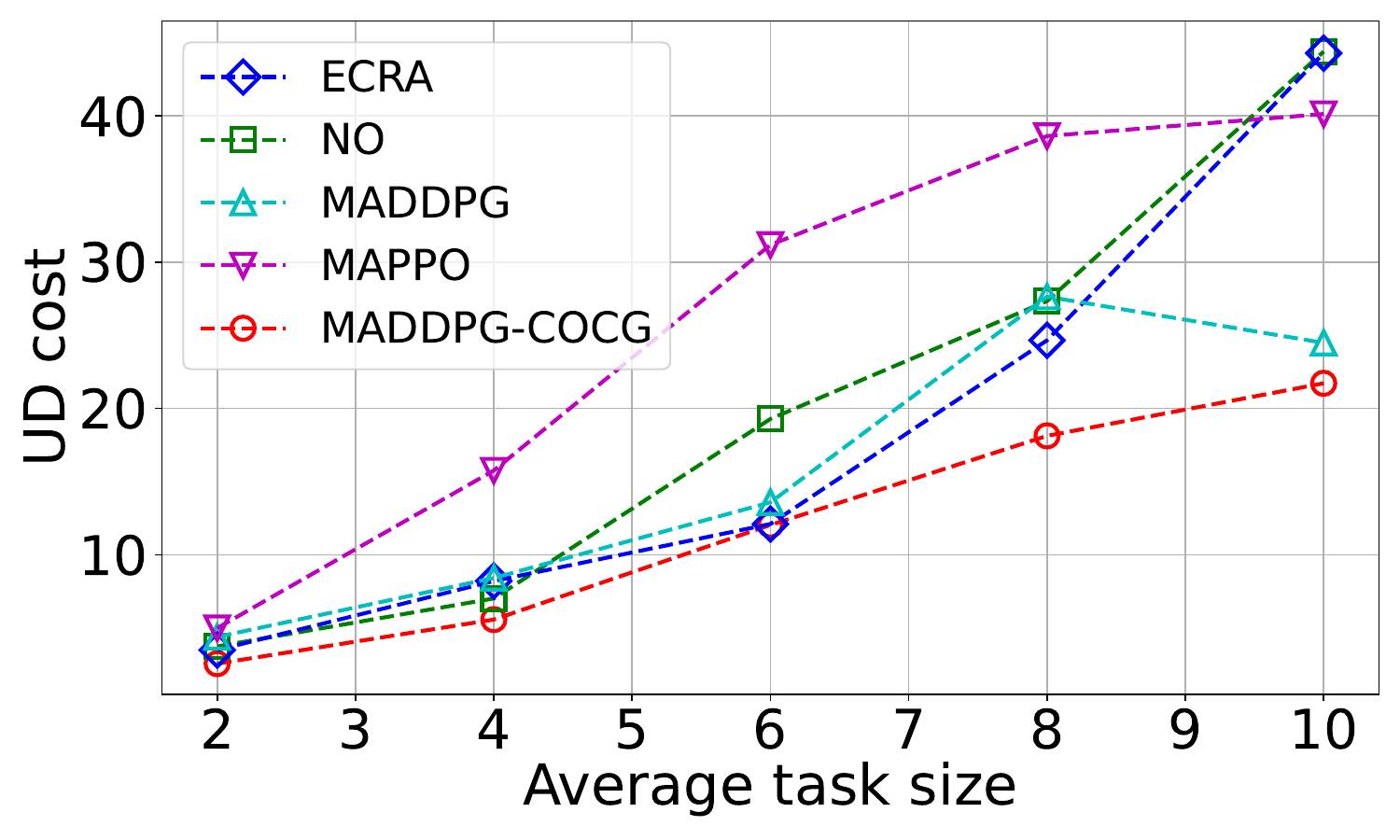}
		\end{minipage}
	}
	\subfigure[Average task completion delay]
	{
		\begin{minipage}[t]{0.23\linewidth}
			\centering
			\includegraphics[width=1.75in]{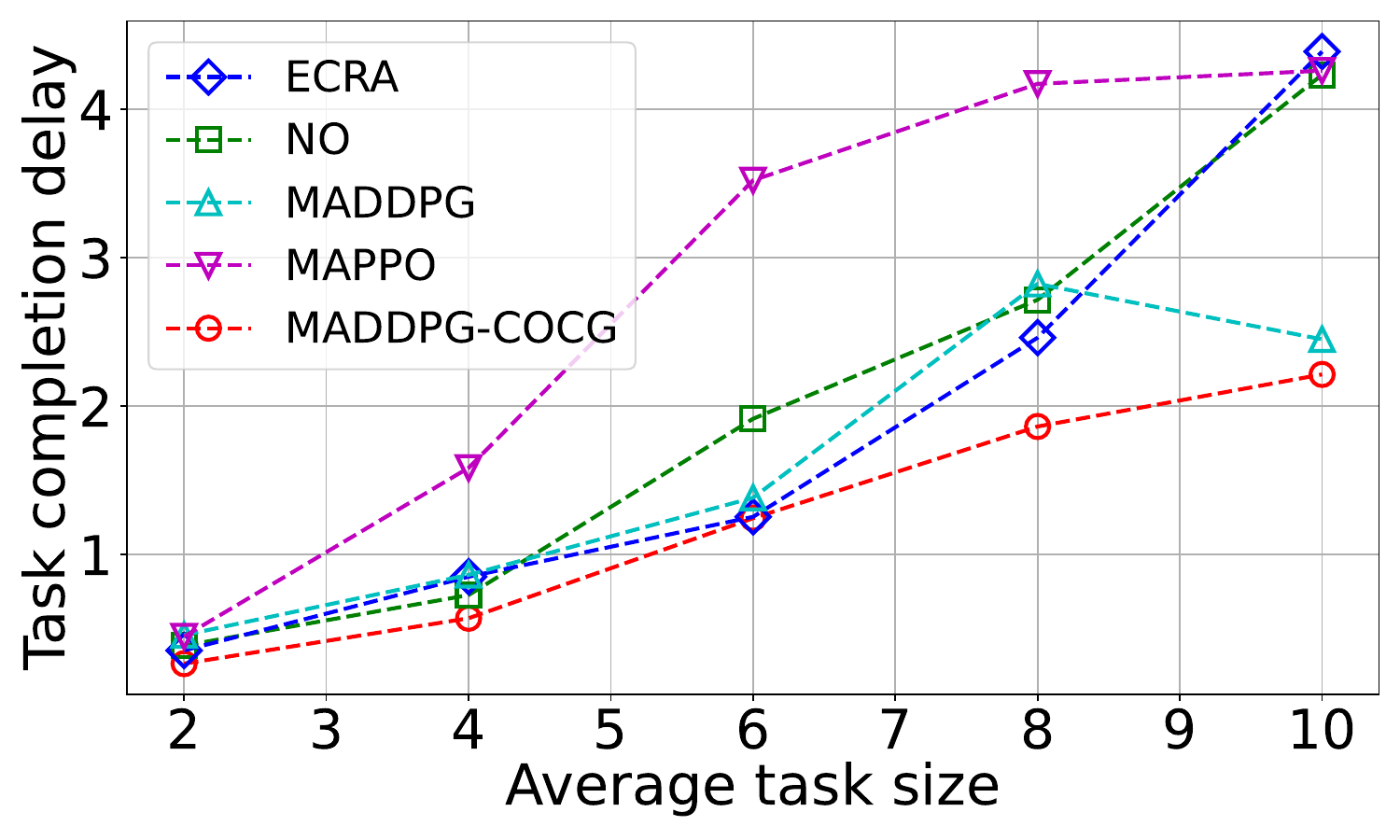}	
		\end{minipage}
	}
	\subfigure[Average UD energy consumption]
	{
		\begin{minipage}[t]{0.23\linewidth}
			\centering
			\includegraphics[width=1.75in]{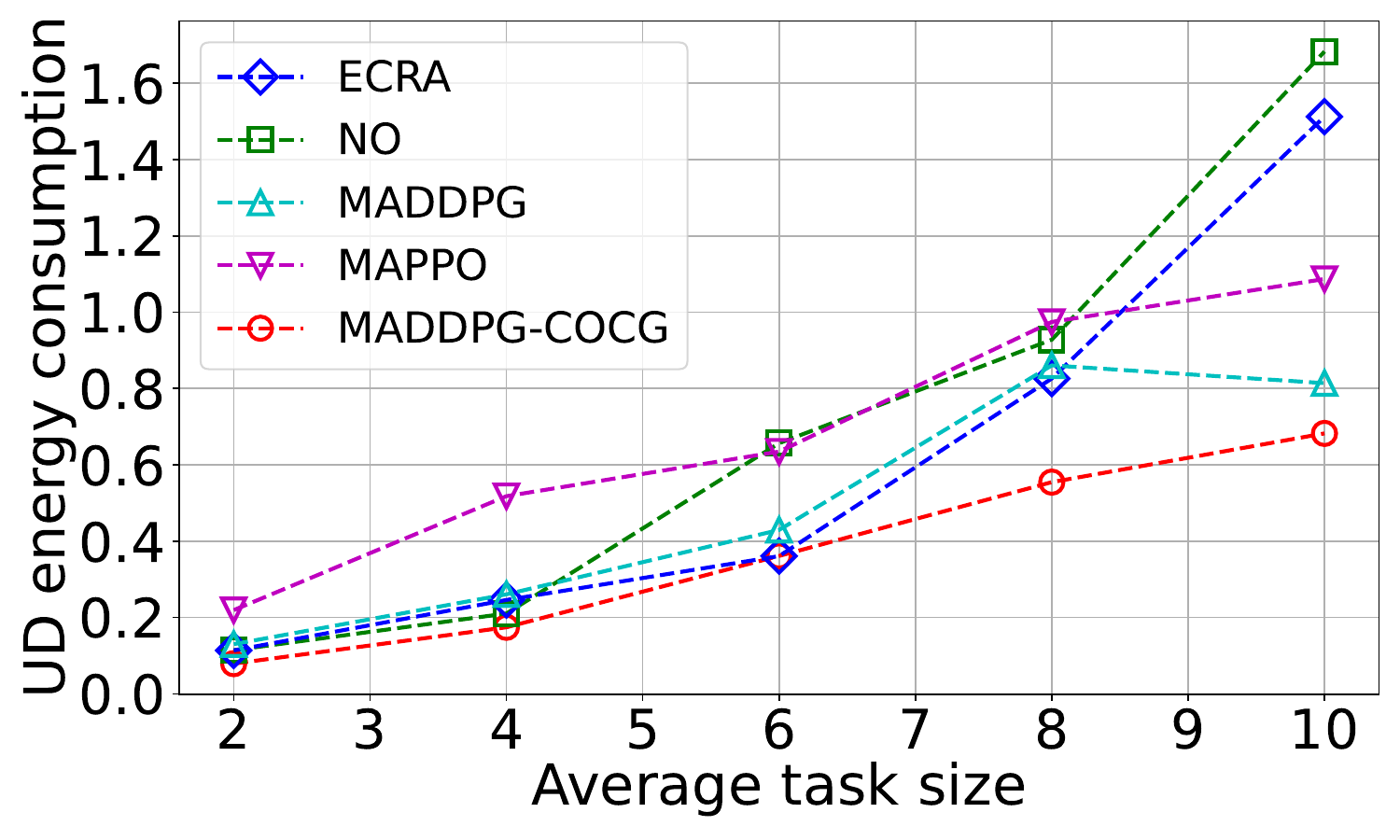}
		\end{minipage}
	}
        \subfigure[Average UAV energy consumption]
	{
		\begin{minipage}[t]{0.23\linewidth}
			\centering
			\includegraphics[width=1.75in]{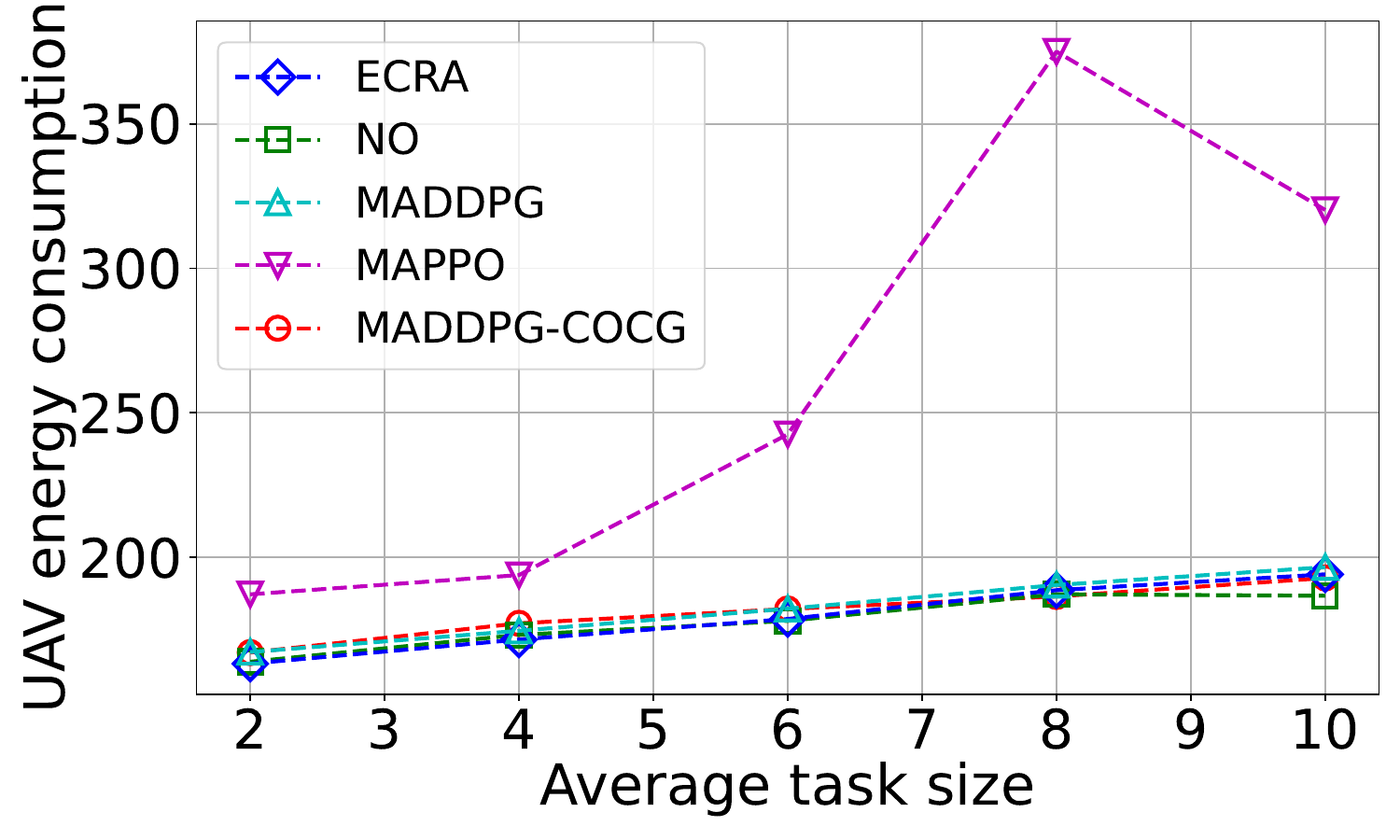}
		\end{minipage}
	}
	\centering
	\caption{System performance with average task size.}
	\label{fig_task}
	\vspace{-1.5em}
\end{figure*}

\par \textbf{Impact of Average Task Size.}
Figs. \ref{fig_task}(a)-(d) illustrate the impact of average task size on aggregated UD cost, average task completion delay, and the energy consumption of UDs and the UAV. As observed in the figures, all these metrics exhibit a clear upward trend for all algorithms as the task data volume increases. This is an expected outcome, as heavier workloads naturally lead to higher overheads in communication, computation, and energy for both UDs and the UAV. However, MAPPO and MADDPG show particularly inferior performance and higher sensitivity to task size, with steep growth and notable instability in task delay and UAV energy. This is attributed to the difficulty their policies have in adapting to high-load MEC scenarios due to complex, high-dimensional action spaces. Furthermore, the performance of ECRA and NO also degrades; NO's nearest offloading becomes suboptimal under high traffic, while ECRA's fixed resource allocation fails to adapt to varying task demands.

\par In contrast, our proposed MADDPG-COCG algorithm consistently outperforms all benchmarks on UD cost, task delay, and UD energy consumption. This effectiveness is due to its synergistic components: coalitional game-based association balances the task load, closed-form resource allocation adapts to varying workloads, and the MADDPG policy dynamically adjusts offloading and trajectory to meet UD demands. Although MADDPG-COCG's UAV energy consumption is slightly higher than NO's for large tasks (over 8 MB), this is a deliberate trade-off. The additional UAV energy is productively spent on flight and computation to significantly reduce service delay and enhance UD energy efficiency. Crucially, MADDPG-COCG exhibits the most stable performance across all metrics. In summary, the results show that our algorithm maintains low delay and high UD energy efficiency in high-load environments, at the cost of a minor increase in UAV energy consumption.

\begin{figure*}[!hbt] 
	\centering
	\setlength{\abovecaptionskip}{2pt}%
	\setlength{\belowcaptionskip}{2pt}%
	\subfigure[Aggregated UD cost]
	{
		\begin{minipage}[t]{0.23\linewidth}
			\raggedleft
			\includegraphics[width=1.75in]{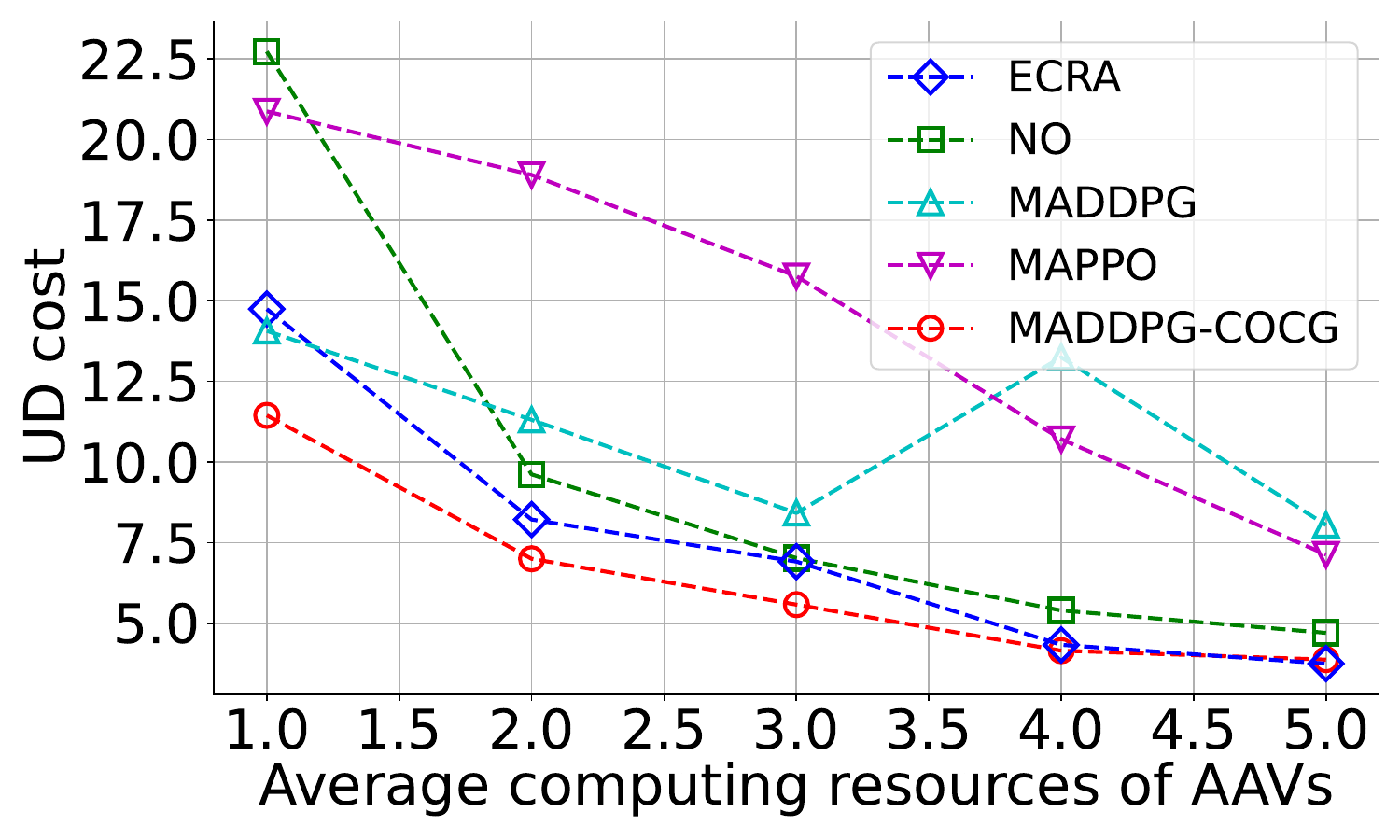}
		\end{minipage}
	}
	\subfigure[Average task completion delay]
	{
		\begin{minipage}[t]{0.23\linewidth}
			\centering
			\includegraphics[width=1.75in]{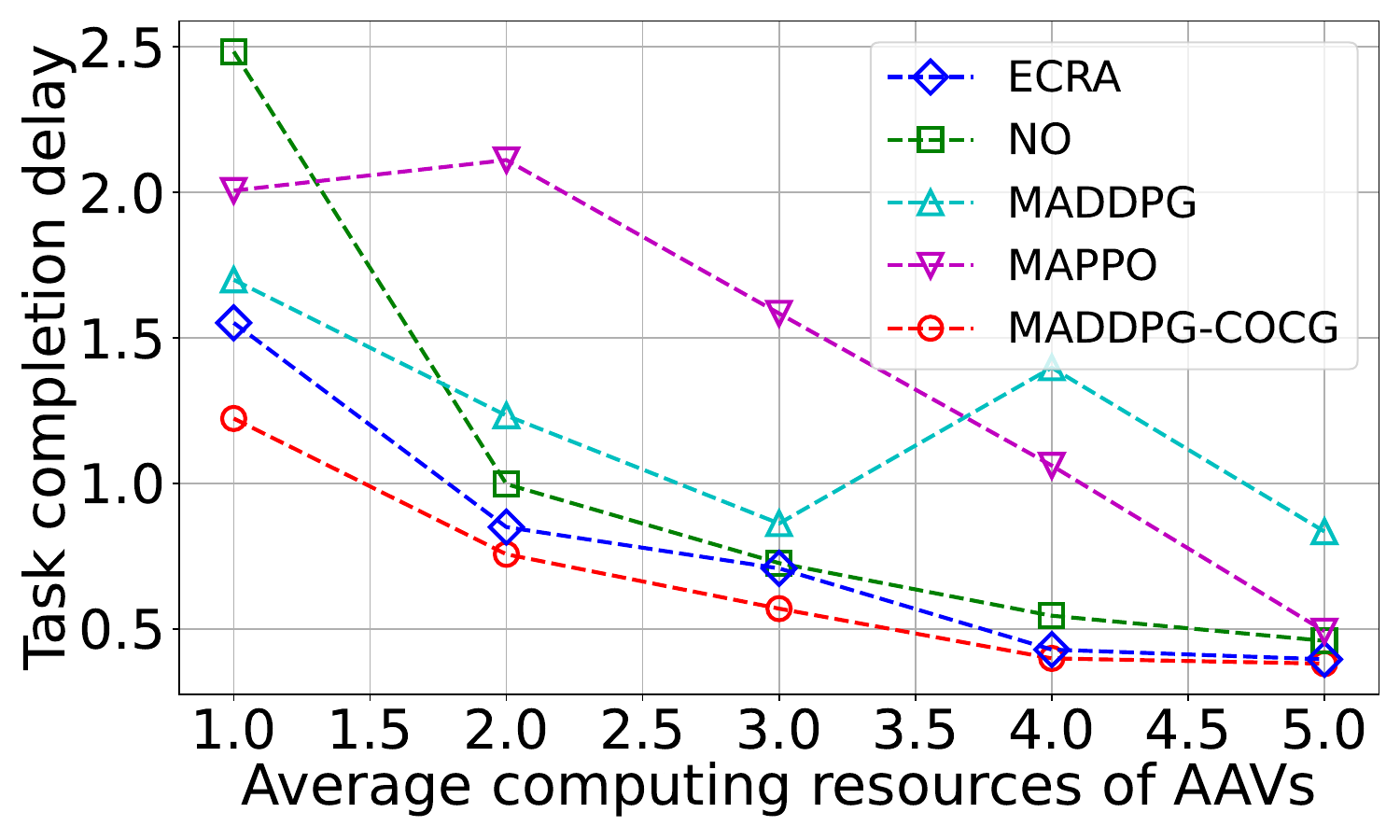}	
		\end{minipage}
	}
	\subfigure[Average UD energy consumption]
	{
		\begin{minipage}[t]{0.23\linewidth}
			\centering
			\includegraphics[width=1.75in]{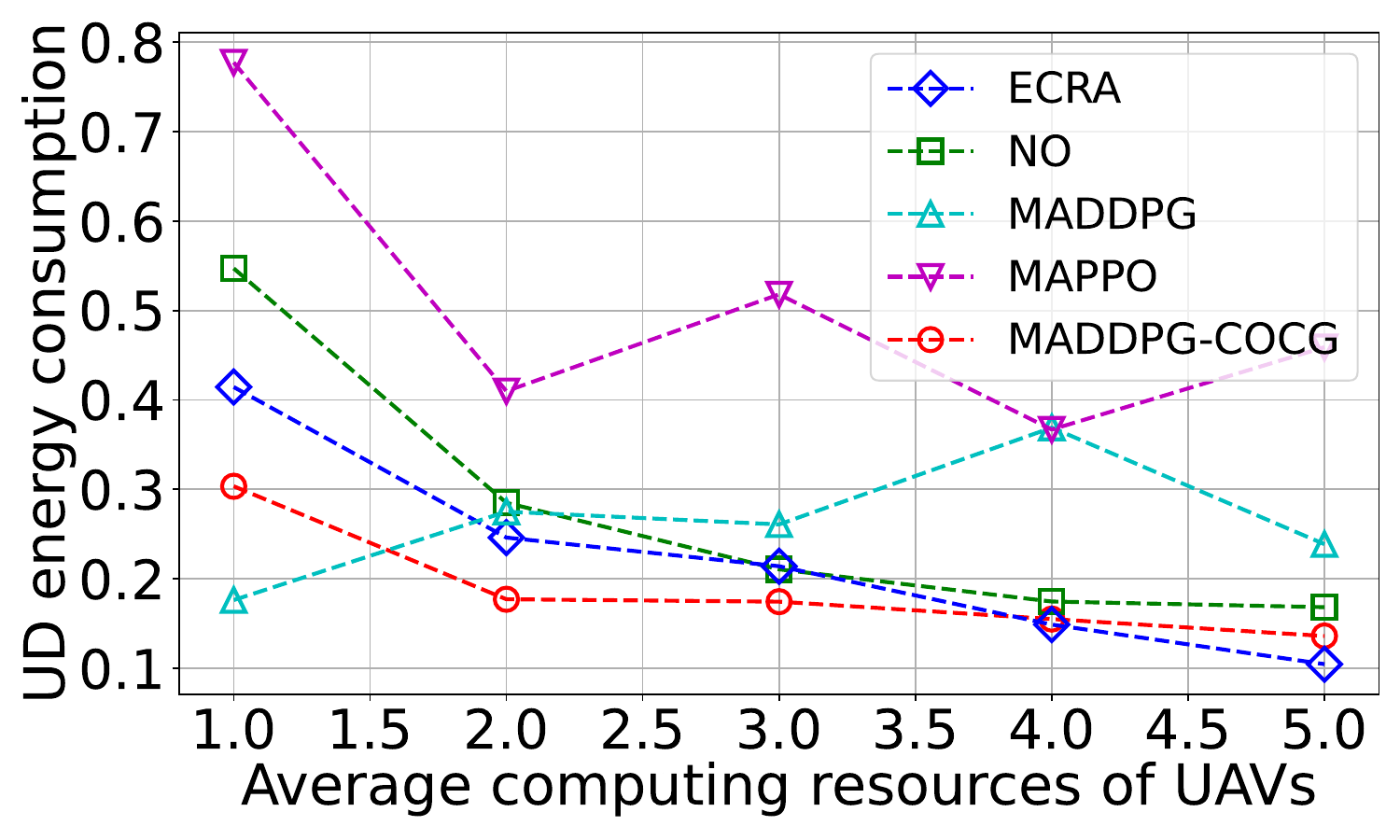}
		\end{minipage}
	}
        \subfigure[Average UAV energy consumption]
	{
		\begin{minipage}[t]{0.23\linewidth}
			\centering
			\includegraphics[width=1.75in]{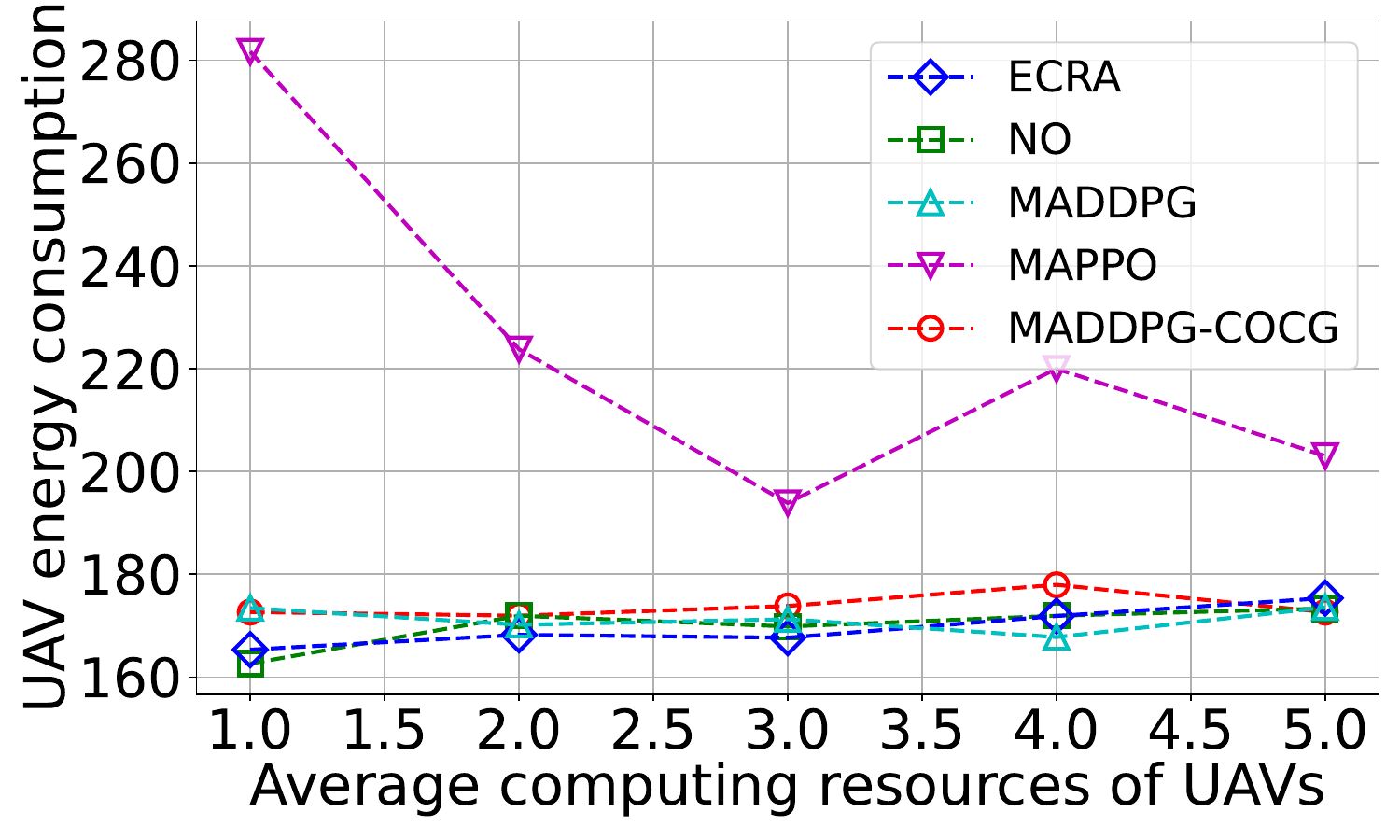}
		\end{minipage}
	}
	\centering
	\caption{System performance with average computing resources of UAVs.}
	\label{fig_res}
	\vspace{-1.5em}
\end{figure*}

\par \textbf{Impact of Average UAV Computing Resources.} Figs. \ref{fig_res}(a)-(d) show how UAV computing resources affect system performance. Initially, as UAV computing power increases, metrics like UD cost, task delay, and UD energy consumption for most algorithms show a decreasing trend. This is because greater processing capacity allows the UAV to handle more offloaded tasks efficiently, directly reducing computation delays and saving UD battery life. However, this improvement is not limitless. As resources become abundant, more UDs are incentivized to offload, causing the system's performance bottleneck to shift from computation to the UAV's energy budget. Consequently, the performance curves begin to stabilize. In contrast, MADDPG and MAPPO exhibit low and unstable performance. Their policies struggle to effectively manage the trade-off between leveraging increased computational power and conserving limited UAV energy, leading to poor convergence and high sensitivity to resource variations.

\par Our proposed MADDPG-COCG algorithm, however, demonstrates superior performance and stability. It not only follows the beneficial decreasing trend but also achieves a significantly lower and more stable plateau than all benchmarks. This is because its hybrid design intelligently manages the shifting system dynamics. The coalitional game for UD association optimizes the offloading decisions, preventing the UAV from being overwhelmed, while the convex optimization module ensures a near-optimal allocation of the available computing resources. Therefore, even when UAV energy becomes the primary constraint, our algorithm finds a more efficient equilibrium, balancing the gains from powerful computation with sustainable energy expenditure. This leads to the best overall performance in UD cost, delay, and energy efficiency, confirming the algorithm's robustness and adaptability to varying resource availability.

\par Compared to the benchmarks, the proposed MADDPG-COCG algorithm reduces UD costs by 26.10\% to 40.44\%, decreases task completion delay by 18.62\% to 49.64\%, and lowers UD energy consumption by 21.1\% to 50.76\% when in highly resource-constrained conditions. This superior performance stems from several key factors. First, the COCG method for UD association and computing resource allocation not only reduces the action space of the problem, but also enables dynamic and balanced task distribution, along with fast and accurate adaptation to workload variations. Moreover, the MADDPG effectively optimizes continuous decisions of task offloading ratios and UAV trajectory planning under a reduced action space, thus enhancing learning stability and decision efficiency. Although MADDPG-COCG incurs slightly higher UAV energy consumption under sufficient computing resources compared to NO, this is a reasonable compromise to achieve  improvements in both delay and UD energy efficiency.
\vspace{-1.0em}
\begin{figure}[!hbt] 
	\centering
         \setlength{\abovecaptionskip}{2pt}%
	\setlength{\belowcaptionskip}{2pt}%
	\includegraphics[width=3.3in]{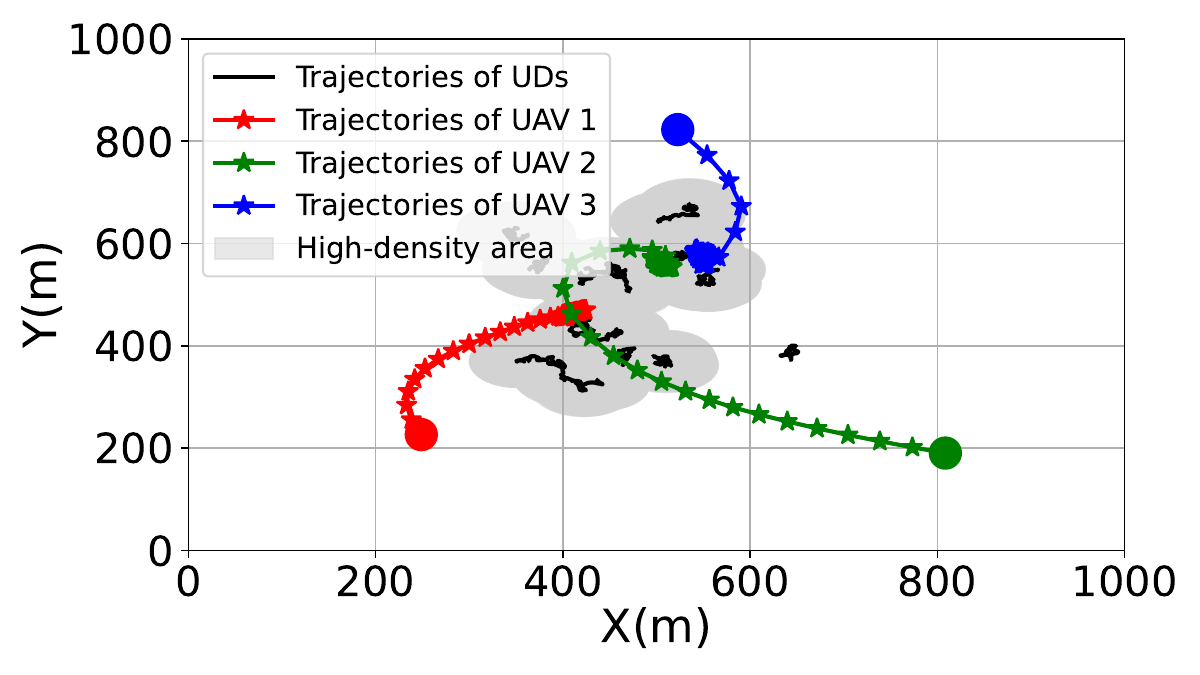}
	\caption{Trajectories of the UDs and UAVs.}
	\label{fig_tra}
       \vspace{-0.5em}
\end{figure}

\subsubsection{Trajectory Results}  

\par Fig. \ref{fig_tra} visualizes the learned UAV trajectories, which dynamically adapt to the spatial distribution of UDs. The UAVs navigate towards areas with high data demand to provide better communication conditions for more users. Notably, they travel at high speeds when traversing sparse areas and then decelerate upon reaching UD clusters to accommodate offloading requests, all while maintaining safe separation distances to prevent collisions. These results demonstrate that the proposed MADDPG-COCG algorithm plans trajectories that effectively balance service efficiency with flight safety.

%
%

\section{Conclusion}
\label{sec_conclusion}

\par In this work, we have addressed the UD cost minimization problem in SAGIN-MEC systems for LAE. We have proposed a hierarchical SAGIN-MEC architecture that integrates space, air, and ground layers to provide ubiquitous computing services. Moreover, we have formulated the UCMOP to minimize the UD costs by optimizing task offloading ratio, UAV trajectory planning, computing resource allocation, and UD association. To solve the formulated NP-hard problem, we propose the MADDPG-COCG algorithm, which leverages MADDPG to address the dynamic and heterogeneous characteristics of the SAGIN-MEC environment, while employing the COCG method to enhance the MADDPG by handling the hybrid and dimension-varying decision variables. Simulation results validate the effectiveness of our proposed MADDPG-COCG algorithm. Specifically, the MADDPG-COCG algorithm significantly outperforms the key benchmarks in terms of the user-centric performance indicators of the aggregated UD cost, task completion delay, and UD energy consumption, with a slight increase in UAV energy consumption. The MADDPG-COCG algorithm also demonstrates superior convergence, stability, and scalability, thus confirming its robustness and applicability for SAGIN-MEC systems.

\ifCLASSOPTIONcaptionsoff
\newpage
\fi

\bibliographystyle{IEEEtran}
\bibliography{references.bib}

\begin{IEEEbiography}[{\includegraphics[width=1in,height=1.23in,clip,keepaspectratio]{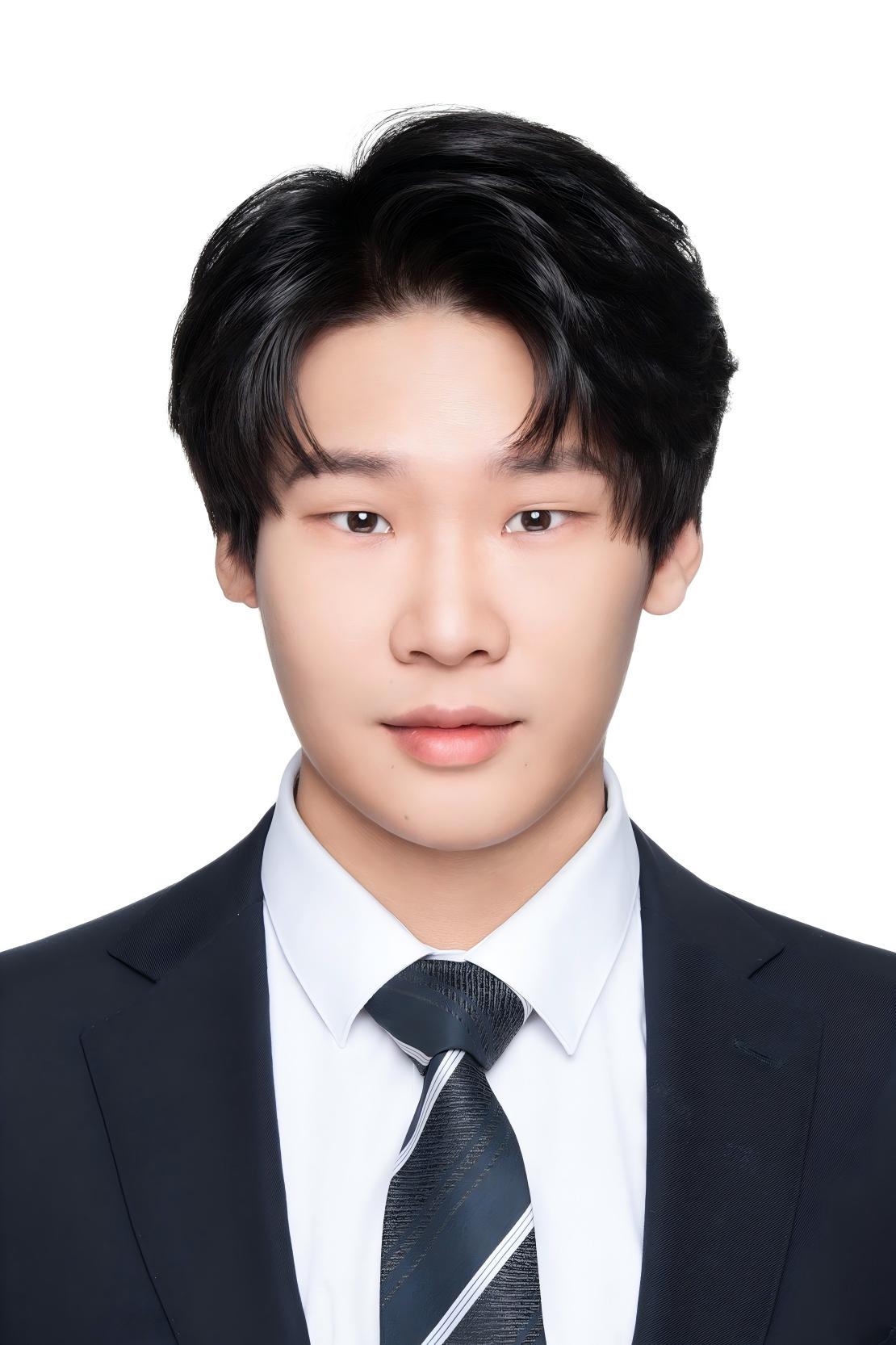}}]{Weihong Qin} received a BS degree in Software Engineering from Jilin University, Changchun, China, in 2023. He is currently working toward the Ph.D. degree in Computer Science and Technology at Jilin University, Changchun, China. His research interests include mobile edge computing and optimizations.
\end{IEEEbiography}

\begin{IEEEbiography}[{\includegraphics[width=1in,height=1.23in,clip,keepaspectratio]{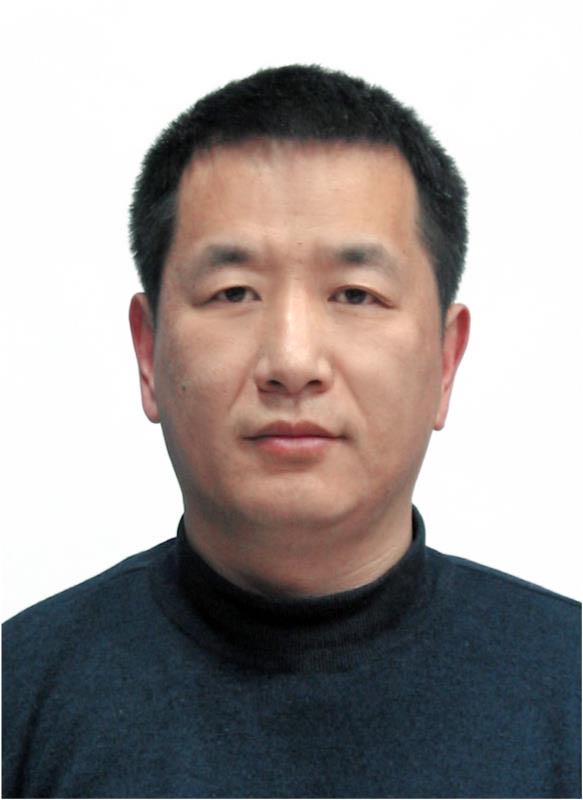}}]{Aimin Wang}received the Ph.D. degree in communication and information system from Jilin University,  Changchun, China, in 2004. He is currently a Professor with Jilin University. His research interests are wireless sensor networks and QoS for multimedia transmission.
\end{IEEEbiography}

 \begin{IEEEbiography}[{\includegraphics[width=1in,height=1.23in,clip,keepaspectratio]{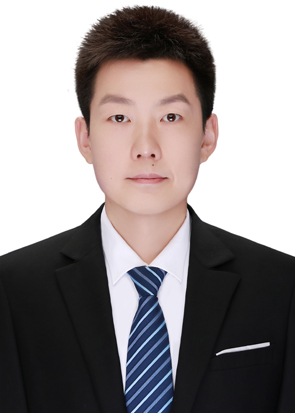}}]{Geng Sun} (Senior Member, IEEE) received the B.S. degree in communication engineering from Dalian Polytechnic University, and the Ph.D. degree in computer science and technology from Jilin University, in 2011 and 2018, respectively. He was a Visiting Researcher with the School of Electrical and Computer Engineering, Georgia Institute of Technology, USA. He is a Professor in the College of Computer Science and Technology at Jilin University. Currently, he is working as a visiting scholar at the College of Computing and Data Science, Nanyang Technological University, Singapore. He has published over 100 high-quality papers, including IEEE TMC, IEEE JSAC, IEEE/ACM ToN, IEEE TWC, IEEE TCOM, IEEE TAP, IEEE IoT-J, IEEE TIM, IEEE INFOCOM, IEEE GLOBECOM, and IEEE ICC. He serves as the Associate Editors of IEEE Communications Surveys \& Tutorials, IEEE Transactions on Communications, IEEE Transactions on Vehicular Technology, IEEE Transactions on Network Science and Engineering, IEEE Transactions on Network and Service Management and IEEE Networking Letters. He serves as the Lead Guest Editor of Special Issues for IEEE Transactions on Network Science and Engineering, IEEE Internet of Things Journal, IEEE Networking Letters. He also serves as the Guest Editor of Special Issues for IEEE Transactions on Services Computing, IEEE Communications Magazine, and IEEE Open Journal of the Communications Society. His research interests include Low-altitude Wireless Networks, UAV communications and Networking, Mobile Edge Computing (MEC), Intelligent Reflecting Surface (IRS), Generative AI and Agentic AI, and Deep Reinforcement Learning.

\end{IEEEbiography}

\begin{IEEEbiography}[{\includegraphics[width=1in,height=1.23in,clip,keepaspectratio]{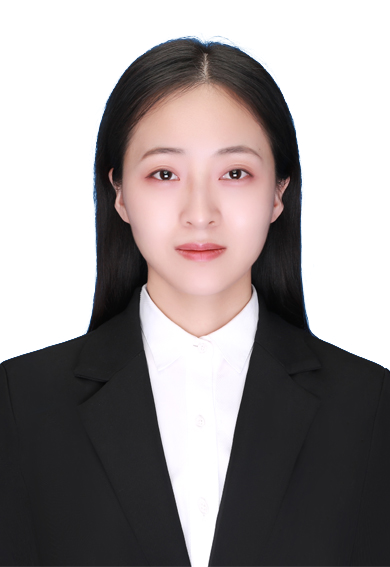}}]{Zemin Sun} (Member, IEEE) received a BS degree in Software Engineering, an MS degree and a Ph.D degree in Computer Science and Technology from Jilin University, Changchun, China, in 2015, 2018, and 2022, respectively. She was a visiting Ph.D. student at the University of Waterloo. She currently serves as an assistant researcher in the College of Computer Science and Technology at Jilin University. Her research interests include vehicular networks, edge computing, and game theory.
\end{IEEEbiography}

\begin{IEEEbiography}[{\includegraphics[width=1in,height=1.23in,clip,keepaspectratio]{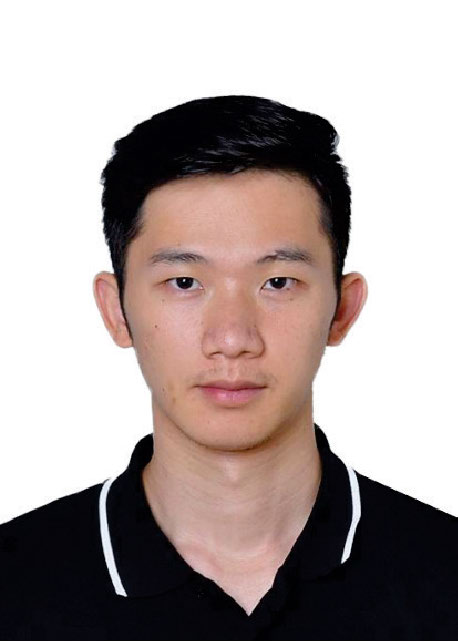}}]{Jiacheng Wang} received the Ph.D. degree from the School of Communication and Information Engineering, Chongqing University of Posts and Telecommunications, Chongqing, China. He is currently a Research Associate in computer science and engineering with Nanyang Technological University, Singapore. His research interests include wireless sensing, semantic communications, and metaverse.
\end{IEEEbiography}

\begin{IEEEbiography}[{\includegraphics[width=1in,height=1.23in,clip,keepaspectratio]{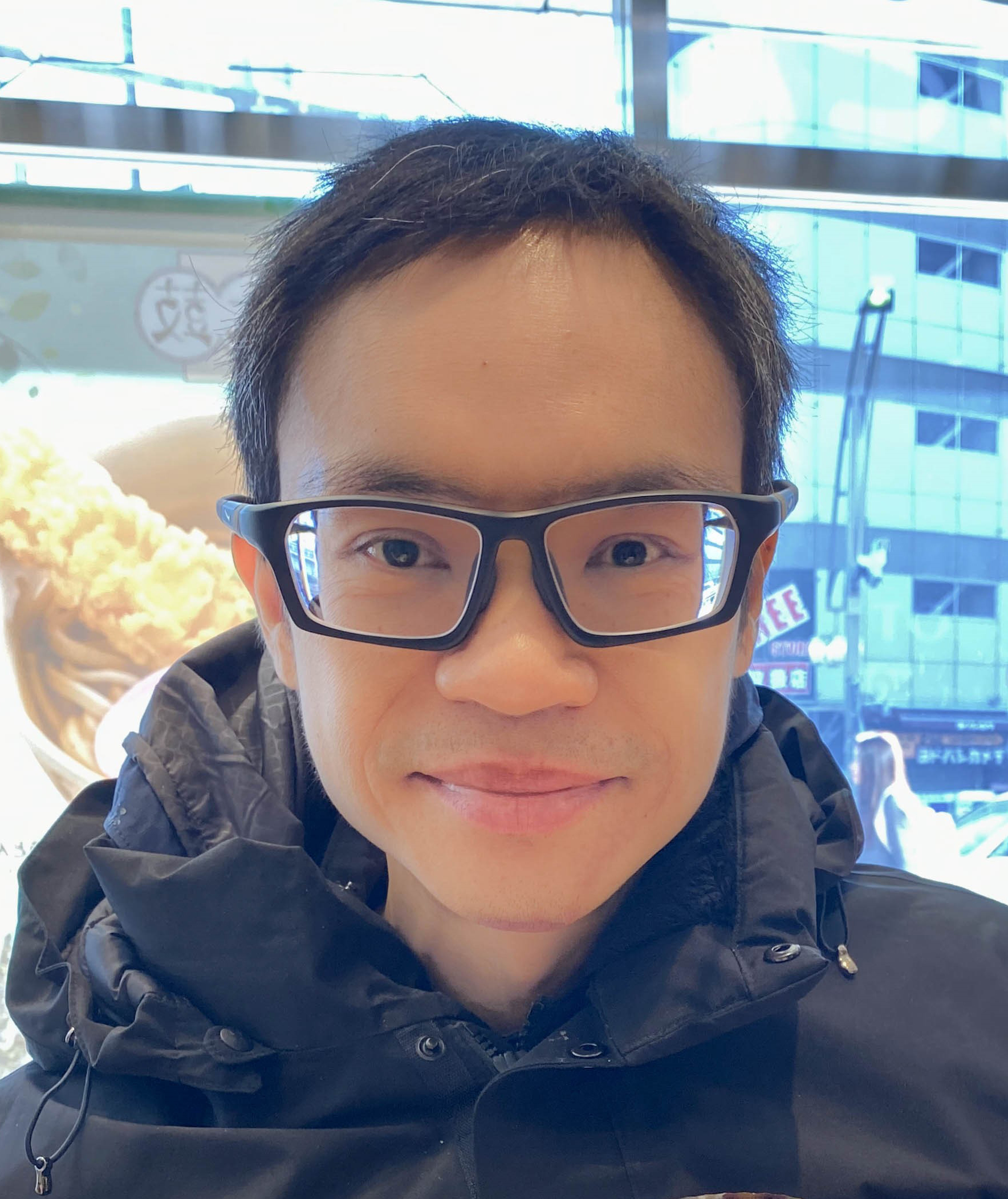}}]{Dusit Niyato} (Fellow, IEEE) received the B.Eng. degree from the King Mongkuts Institute of Technology Ladkrabang (KMITL), Thailand, in 1999, and the Ph.D. degree in electrical and computer engineering from the University of Manitoba, Canada, in 2008. He is currently a Professor with the School of Computer Science and Engineering, Nanyang Technological University, Singapore. His research interests include the Internet of Things (IoT), machine learning, and incentive mechanism design.
\end{IEEEbiography}

\begin{IEEEbiography}[{\includegraphics[width=1in,height=1.23in,clip,keepaspectratio]{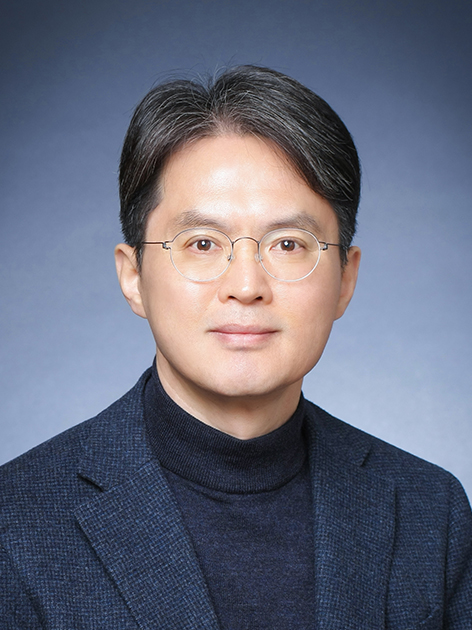}}]{Dong In Kim} (Fellow, IEEE) received the Ph.D. degree in electrical engineering from the Uni
versity of Southern California, Los Angeles, CA, USA, in 1990. He was a Tenured Professor with the School of Engineering Science, Simon Fraser University, Burnaby, BC, Canada. He is currently a Distinguished Professor with the College of Information and Communication Engineering, Sungkyunkwan University, Suwon, South Korea. He is a Fellow of the Korean Academy of Science and Technology and a Member of the National Academy of Engineering of Korea. He was the first recipient of the NRF of Korea Engineering Research Center in Wireless Communications for RF Energy Harvesting from 2014 to 2021. Hereceived several research awards, including the 2023 IEEE ComSoc Best Survey Paper Award and the 2022 IEEE Best Land Transportation Paper Award. He was selected the 2019 recipient of the IEEE ComSoc Joseph LoCicero Award for Exemplary Service to Publications. He was the General Chair of the IEEE ICC 2022, Seoul. Since 2001, he has been serving as an Editor, an Editor at Large, and an Area Editor of Wireless Communications I for IEEE Transactions on Communications. From 2002 to 2011, he served as an Editor and a Founding Area Editor of Cross-Layer Design and Optimization for IEEE Transactions on Wireless Communications. From 2008 to 2011, he served as the CoEditor-in-Chief for the IEEE/KICS Journal of Communications and Networks. He served as the Founding Editor-in-Chief for the IEEE Wireless Communications Letters from 2012 to 2015. He has been listed as a 2020/2022 Highly Cited Researcher by Clarivate Analytics.
\end{IEEEbiography}

\begin{IEEEbiography}[{\includegraphics[width=1in,height=1.23in,clip,keepaspectratio]{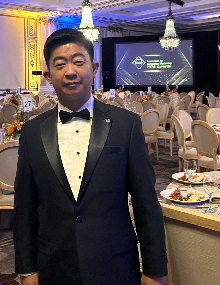}}]{Zhu Han} (Fellow, IEEE) received the B.S. degree in electronic engineering from Tsinghua University, in 1997, and the M.S. and Ph.D. degrees in electrical and computer engineering from the University of Maryland, College Park, in 1999 and 2003, respectively. From 2000 to 2002, he was an R\&D Engineer of JDSU, Germantown, Maryland. From 2003 to 2006, he was a Research Associate at the University of Maryland. From 2006 to 2008, he was an assistant professor at Boise State University, Idaho. Currently, he is a John and Rebecca Moores Professor in the Electrical and Computer Engineering Department as well as in the Computer Science Department at the University of Houston, Texas. Dr. Han’s main research targets on the novel game-theory related concepts critical to enabling efficient and distributive use of wireless networks with limited resources. His other research interests include wireless resource allocation and management, wireless communications and networking, quantum computing, data science, smart grid, carbon neutralization, security and privacy. Dr. Han received an NSF Career Award in 2010, the Fred W. Ellersick Prize of the IEEE Communication Society in 2011, the EURASIP Best Paper Award for the Journal on Advances in Signal Processing in 2015, IEEE Leonard G. Abraham Prize in the field of Communications Systems (best paper award in IEEE JSAC) in 2016, IEEE Vehicular Technology Society 2022 Best Land Transportation Paper Award, and several best paper awards in IEEE conferences. Dr. Han was an IEEE Communications Society Distinguished Lecturer from 2015 to 2018 and ACM Distinguished Speaker from 2022 to 2025, AAAS fellow since 2019, and ACM Fellow since 2024. Dr. Han is a 1\% highly cited researcher since 2017 according to Web of Science. Dr. Han is also the winner of the 2021 IEEE Kiyo Tomiyasu Award (an IEEE Field Award), for outstanding early to mid-career contributions to technologies holding the promise of innovative applications, with the following citation: ``for contributions to game theory and distributed management of autonomous communication networks."
\end{IEEEbiography}

\end{document}